\documentclass[11pt]{article}

\usepackage[utf8]{inputenc}
\usepackage[T1]{fontenc}
\usepackage{lmodern}
\usepackage{amsmath,amssymb,amsthm,amsfonts}
\usepackage{mathtools}
\usepackage{microtype}
\usepackage{enumitem}
\usepackage{geometry}
\usepackage{hyperref}
\usepackage{bm}
\usepackage{dsfont}
\usepackage{tikz-cd}
\usepackage{cite}

\usepackage{mathrsfs}
\usepackage{bbm}

\newcommand{\Xcal}{\mathcal{X}}
\newcommand{\Ycal}{\mathcal{Y}}
\newcommand{\Pcal}{\mathcal{P}}
\newcommand{\Mfrak}{\mathfrak{M}}
\newcommand{\ThetaMan}{\Theta}

\setlist{nosep}

\newtheorem{theorem}{Theorem}[section]
\newtheorem{proposition}[theorem]{Proposition}

\newtheorem{corollary}[theorem]{Corollary}
\theoremstyle{definition}
\newtheorem{definition}[theorem]{Definition}
\newtheorem{example}[theorem]{Example}
\theoremstyle{remark}
\newtheorem{remark}[theorem]{Remark}

\newcommand{\RR}{\mathbb{R}}
\newcommand{\NN}{\mathbb{N}}
\newcommand{\PP}{\mathbb{P}}
\newcommand{\EE}{\mathbb{E}}

\newcommand{\eps}{\varepsilon}

\title{Self-Improving AI Agents through Self-Play}
\author{Przemyslaw Chojecki \\ \small ulam.ai}
\date{December 2, 2025}

\begin{document}
\maketitle

\begin{abstract}
We extend the moduli-theoretic framework of psychometric batteries \cite{Chojecki2025} to the domain of dynamical systems. While previous work established the AAI capability score as a static functional $\Phi_{\mathcal{B}}$ on the space of agent representations $\mathcal{P}(X_{\mathcal{B}})$, this paper formalizes the agent as a flow $\nu_r$ parameterized by computational resource $r$, governed by a recursive \textbf{Generator-Verifier-Updater (GVU)} operator. We prove that this operator generates a vector field on the parameter manifold $\Theta$, and we identify the coefficient of self-improvement $\kappa$ as the Lie derivative of the capability functional along this flow.

The central contribution of this work is the derivation of the \textbf{Variance Inequality}, a spectral condition that is sufficient (under mild regularity) for the stability of self-improvement. We show that a sufficient condition for $\kappa > 0$ is that, up to curvature and step-size effects, the combined noise of generation and verification must be small enough.

We then apply this formalism to unify the recent literature on Language Self-Play (LSP), Self-Correction, and Synthetic Data bootstrapping. We demonstrate that architectures such as STaR \cite{zelikman2022star}, SPIN \cite{chen2024spin}, Reflexion \cite{shinn2023reflexion}, GANs and AlphaZero are not merely heuristics but specific topological realizations of the GVU operator that satisfy the Variance Inequality through filtration, adversarial discrimination, or grounding in formal systems.
\end{abstract}

\section{Introduction}\label{sec:intro}

The central problem in Artificial General Intelligence (AGI) is not the achievement of a specific benchmark score, but the achievement of \emph{ignition}: the point at which an agent can autonomously convert computational resources into capability gains without human intervention. In the framework of \cite{Chojecki2025}, we defined the capability of an agent $\mathcal{A}$ on a battery $\mathcal{B}$ as a functional value $\Phi_{\mathcal{B}}(\rho_{\mathcal{B}}(\mathcal{A}))$. However, for current Large Language Models (LLMs), this value is static once pre-training concludes. As noted in \cite{xu2023expert}, the trajectory of self-improvement for standard LLMs is flat ($\kappa \approx 0$) or decaying due to hallucination drift.

By contrast, systems like AlphaGo Zero \cite{silver2017mastering} exhibited $\kappa \gg 0$, reaching superhuman capability solely through self-play. The disparity lies in the nature of the \emph{verification signal}. Go provides a noiseless, ground-truth verifier (the game rules). Open-ended domains do not.

To bridge this gap, recent literature has proposed various mechanisms for "self-correction" and "self-play" in language models. These include iterative reasoning bootstrapping (STaR \cite{zelikman2022star}), zero-sum language games (SPIN \cite{chen2024spin}, LSP \cite{hritu2025lsp}), and verbal reinforcement learning (Reflexion \cite{shinn2023reflexion}), but also GANs and AlphaZero.

This paper unifies these approaches under a single rigorous mathematical framework. We define the \textbf{GVU Operator} as the canonical engine of self-improvement. We show that the success or failure of any self-improving agent is determined by the spectral properties of this operator acting on the tangent bundle of the moduli space. Specifically, we derive a "Second Law of AGI Dynamics": 
\emph{Entropy (hallucination) tends to increase unless the combined signal from generation and
verification is strong enough, relative to their noise and to
curvature, to keep the expected capability gain positive. In practice,
many architectures satisfy this by making verification spectrally
``easier'' than generation (e.g., via oracles, ensembles, or
external structure).}

Central message for practitioners is: The Variance Inequality tells you exactly why your RL training plateaus and what to do about it - strengthen the verifier, not the generator. Check out \ref{sec:llm} for relation of our framework to current LLM training pipelines.

\subsection*{Contributions}

This paper makes the following contributions:

\begin{itemize}
    \item \textbf{From static scores to dynamical flows.}
    We extend the moduli-theoretic framework of psychometric batteries \cite{Chojecki2025} from static capability scores to dynamical trajectories. An agent is modeled as a flow $(\nu_r)_{r \ge 0}$ on a statistical parameter manifold $\ThetaMan$, and the self-improvement coefficient $\kappa(r)$ is identified with the Lie derivative of the capability functional $F = \Phi_{\mathcal{B}} \circ \rho_{\mathcal{B}}$ along this flow. This yields an operational notion of \emph{ignition} as sustained $\kappa > 0$ across capability fibers.

    \item \textbf{The GVU operator and a universality theorem.}
    We formalize the \emph{Generator--Verifier--Updater} (GVU) operator $\mathcal{T}_{\mathrm{GVU}} = \mathcal{U} \circ \mathcal{V} \circ \mathcal{G}$ as the canonical engine of self-improvement, and prove a score-based GVU representation theorem: on a regular statistical manifold, any first-order, sample-based update vector field can be written in REINFORCE form
    \[
    v(\theta)
    =
    \EE_{(x,y)\sim \mu \otimes \pi_\theta}
    \big[V_\theta(x,y)\,\nabla_\theta \log \pi_\theta(y\mid x)\big]
    \]
    for some scalar potential $V_\theta$. Thus any rational, data-driven self-update implicitly instantiates a GVU with an internal Verifier potential. A non-trivial verifier is shown to be \emph{necessary} for non-zero expected $\kappa$.

    \item \textbf{The Variance Inequality and the Hallucination Barrier.}
    We derive the \emph{Variance Inequality}, a sufficient spectral condition for expected capability gain $\EE[\Delta F] > 0$. It quantitatively relates alignment $\rho$ between the internal potential and the external score, generation and verification variances $(\sigma_{\mathcal{G}}^2,\sigma_{\mathcal{V}}^2)$, curvature $L$, and step size $\eta$. A corollary identifies the \emph{Hallucination Barrier}: in diagonal regimes where $\mathcal{V} \approx \mathcal{G}$, verification noise matches generation noise and self-correction typically fails to produce sustained $\kappa>0$.

    \item \textbf{Geometric and spectral design levers.}
    Working on the Fisher-information statistical manifold $(\ThetaMan,g)$, we interpret the GVU drift as a noisy vector field whose usefulness is governed by the Fisher angle between the mean update and the true gradient of $F$. We analyze generic design levers that improve $\kappa$---ensemble verifiers, group-based normalization (GRPO-style schemes), oracle-like executors (code, games, proofs), and ``cold'' verifier interfaces in diagonal GVU---and quantify how they increase $\mathrm{SNR}(\mathcal{V})$ and widen the stable stepsize window. We also introduce a Goodhart-type limit on long-run $\kappa$ via decay of the alignment coefficient $\rho$ under proxy optimization.

    \item \textbf{Topological realizations and an empirical $\hat{\kappa}$ protocol.}
    We show that a wide range of existing self-improvement methods---AlphaZero, GANs, STaR, SPIN/LSP, PRMs, RAG self-training, self-debugging code agents, RLHF, Constitutional AI, Self-Instruct, and GRPO---are concrete topological realizations of the GVU operator on different fibers (Sociality, Planning, Embodiment, Recursive, Alignment, Synthetic, Critic-less) of the moduli space. Finally, we propose a finite-difference evaluation protocol for estimating an empirical self-improvement rate $\hat{\kappa}$ from before/after battery scores under a fixed compute budget.
\end{itemize}

\paragraph{Relation to Reinforcement Learning.}
Classical reinforcement learning (RL) provides a formal model for agents that optimize a reward signal in a Markov decision process. Our framework recovers this setting as a special case but is designed to encompass a much broader class of self-improving systems, including LLM-based agents trained purely with supervised or synthetic data rather than explicit RL. In standard RL, an environment and reward function define a single objective $J(\pi_\theta)$, and updates are derived from estimates of the policy gradient $\nabla_\theta J(\pi_\theta)$. In our setting, the \emph{battery} $\mathcal{B}$ and capability functional $\Phi_{\mathcal{B}}$ play the role of a generalized evaluation layer, inducing a scalar capability $F(\theta) = (\Phi_{\mathcal{B}}\circ\rho_{\mathcal{B}})(\theta)$ that can aggregate performance across heterogeneous tasks and modalities. The Generator--Verifier--Updater (GVU) operator then subsumes actor--critic and self-play schemes as special cases: the generator corresponds to sampling trajectories, the verifier to any internal scoring mechanism (reward model, contrastive critic, oracle, or verbal judge), and the updater to a first-order policy update, possibly implemented via supervised fine-tuning rather than explicit RL. Equipping the parameter manifold $\Theta$ with the Fisher information metric turns it into a statistical manifold, allowing us to express these updates as natural-gradient flows and to derive the Variance Inequality, a spectral condition under which \emph{any} such GVU loop---RL-based or not---yields positive expected capability gain. In this way, the theory applies uniformly to conventional RL agents, RLHF/RLAIF pipelines, and purely SFT-trained LLM agents that self-improve via self-correction, synthetic data, or tool use, providing a single geometric lens on their learning dynamics.

\paragraph{Relation to Geometric Deep Learning.}
Our framework is closely related in spirit to Geometric Deep Learning,
which studies neural architectures constrained by underlying geometric
structure (groups, graphs, manifolds). Rather than imposing geometry
on the input domain, we equip the \emph{space of policies and their
learning dynamics} with a statistical and moduli geometry: the parameter
manifold $(\Theta, G)$ with Fisher metric, and the moduli space
$\Mfrak$ of batteries stratified into capability fibers. The GVU
operator then defines a noisy vector field on $(\Theta, G)$, and the
Variance Inequality constrains which such fields can yield positive
drift of the capability functional $F = \Phi_{\mathcal{B}} \circ
\rho_{\mathcal{B}}$. In this sense, our results can be viewed as a
form of ``geometric deep learning of self-improving agents'': on each
fiber (Sociality, Planning, Embodiment, Alignment), different GVU
topologies (adversarial self-play, filtration, execution oracles,
ensemble judges, GRPO) play the role of geometric inductive biases
that make the self-improvement dynamics spectrally stable.

\paragraph{Beyond LLMs and RL.}
Although many of our examples are phrased in terms of language models and RL fine-tuning, the framework is not limited to them. The only ingredients we require are (i) a parametric generator of behaviour---a map $\Pi_\Theta$ from internal state $\theta\in\Theta$ to a stochastic policy $\pi_\theta$ over outputs, (ii) some form of scoring or evaluation, internal or external, and (iii) an update rule that uses these scores to change $\theta$. This pattern occurs across a wide range of systems that are not usually described as RL: evolutionary and black-box optimizers (where $\theta$ parameterizes a search distribution and the Verifier is a fitness function), deep-guided theorem provers and SAT solvers (where proof checkers play the role of high-SNR verifiers), AutoML and architecture search (where $\theta$ controls a proposal policy over models and hyperparameters), and semi-supervised or self-training pipelines in vision and speech. In all of these cases the Generator--Verifier--Updater (GVU) operator provides a canonical decomposition of the self-improvement loop, and the capability functional $F=\Phi_{\mathcal{B}}\circ\rho_{\mathcal{B}}$ allows us to evaluate progress on batteries of heterogeneous tasks rather than a single reward function. The Variance Inequality then applies unchanged: it constrains when noisy, sample-based updates in these non-RL settings can be expected to yield positive drift in capability, even when there is no explicit MDP or reward signal.

\section{Preliminaries: The Geometric Setting}\label{sec:prelim}

We distinguish between the \emph{External Geometry} (the Battery and Moduli Space defined in \cite{Chojecki2025}) and the \emph{Internal Topology} (the Agent's parameters and architecture). The self-improvement process is a mapping from the latter to the former.

\subsection{Semantic Foundations}
Let $\Sigma$ be a finite alphabet of tokens (e.g., the UTF-8 set or a BPE vocabulary).
Let $\Sigma^*$ denote the Kleene closure of $\Sigma$, the set of all finite sequences (strings) over $\Sigma$.
We equip $\Sigma^*$ with the discrete topology. The domains $\Omega_t$ of tasks are subsets of $\Sigma^*$.

\subsection{The External Geometry: Batteries}
We utilize the exact definition of the battery from \cite{Chojecki2025}.

\begin{definition}[Battery]\label{def:battery}
A \emph{battery} is an octuple
\[
  \mathcal B=(T,\ \mathcal F,\ \mathsf S,\ Q^*,\ \mu,\ \mathsf D,\ \Pi,\ \mathsf R),
\]
where:
\begin{itemize}
  \item $T$ is a finite set of tasks; $\mathcal F=\{F_k\}$ is a partition of $T$ into families.
  \item $\mathsf S=\{S_t:\Omega_t\to[0,1]\}_{t\in T}$ are task-specific scoring maps, where $\Omega_t \subseteq \Sigma^*$ is the domain of valid solution traces for task $t$.
  \item $Q^*:T\to[0,1]$ are task thresholds.
  \item $\mu$ is a sampling law on $T\times\Pi\times\mathsf D$ (tasks, seeds, drifts).
  \item $\mathsf D$ (drifts) and $\Pi$ (seeds) are measurable spaces.
  \item $\mathsf R\cong\mathbb{R}^{d_R}$ are resource coordinates (e.g., time, tokens, cost), recorded nonnegatively.
\end{itemize}
\end{definition}

\begin{definition}[Trace and Observables]\label{def:trace}
Given an agent $\mathcal{A}$ and budgetary constraints from $\mathsf{R}$, an evaluation draws i.i.d.\ samples $(t_i,s_i,\delta_i)\sim\mu$ for $i=1,\dots,n$ and produces \emph{traces}
\[
\omega_i \;=\; \mathrm{Run}\!\left(\mathcal{A};\,t_i,\,s_i,\,\delta_i,\,\mathsf{R}\right).
\]
From each trace $\omega_i$, we derive per-task observables:
\begin{itemize}
    \item \textbf{Quality:} $q(t_i) = S_{t_i}(\omega_i) \in [0,1]$.
    \item \textbf{Strict Success:} $z(t_i) = \mathbb{I}\{q(t_i) \ge Q^*(t_i)\}$.
    \item \textbf{Uninterrupted Action Count:} $a(t_i) \in \mathbb{N}$.
    \item \textbf{Plan Depth:} $d(t_i) \in \mathbb{N}$ (length of the longest executed path of prerequisite actions).
    \item \textbf{Incurred Cost:} $c(t_i) \in \mathbb{R}_{\ge 0}$ (derived from the resource component $r$ of $y$).
\end{itemize}
Axis-specific raw statistics $r_x=r_x(\{\omega_i\}_{i=1}^n)$ are computed by fixed functionals and normalized by calibration maps $\phi_x$ to yield axis scores $x\in[0,1]$.
\end{definition}

From the battery, we derive the explicit interaction spaces:

\begin{definition}[Input and Output Spaces]\label{def:io-spaces}
\begin{enumerate}
   \item The \textbf{Input Space} $\mathcal{X}$ is the disjoint union of task-specific prompt domains $P_t \subseteq \Sigma^*$. We identify it with the set of labeled prompts:
    \[
    \mathcal{X} := \bigsqcup_{t \in T} P_t \;\cong\; \bigcup_{t \in T} (P_t \times \{t\}) \;\subseteq\; \Sigma^* \times T.
    \]
    An element $x \in \mathcal{X}$ is a pair $x = (s, t)$, where $s \in P_t$ is the input prompt and $t$ is the task identifier.
    
    \item The \textbf{Output Space} $\mathcal{Y}$ is the product of the semantic trace space and the non-negative resource cone:
    \[
    \mathcal{Y} := \Sigma^* \times \mathbb{R}_{\ge 0}^{d_R}.
    \]
    An element $y \in \mathcal{Y}$ is a pair $y = (\omega, r)$, where $\omega \in \Sigma^*$ is the generated trace (see Definition \ref{def:trace}) and $r$ is the vector of resources consumed to produce it.
\end{enumerate}
\end{definition}

\subsection{The Internal Topology: Parameter Manifold}
Dynamics require a coordinate system. We define the agent's state space to encompass both its static weights (long-term memory) and its dynamic context (working memory).

\begin{definition}[Parameter Manifold $\ThetaMan$]\label{def:manifold}
Let $W \cong \mathbb{R}^d$ be the space of trainable weights (e.g., Transformer parameters).
Let $E$ be the embedding dimension and $L$ the context window size. Let $\mathcal{H} \subset \bigcup_{k=0}^L \mathbb{R}^{k \times E}$ be the space of context states (e.g., the KV-cache or prompt buffer).
The \emph{Parameter Manifold} is the product space:
\[
\ThetaMan := W \times \mathcal{H}.
\]
We equip $\ThetaMan$ with a Riemannian metric $g$ (typically the Fisher Information Metric), allowing the definition of gradients $\nabla_\theta$. A state $\theta = (w, h) \in \ThetaMan$ completely specifies the agent at an instant $r$.
\end{definition}

\begin{remark}[Statistical manifold]\label{rem:statistical-manifold}
When the Riemannian metric $g$ on $\ThetaMan$ is chosen to be the
Fisher information metric induced by the policy family
$\{\pi_\theta\}_{\theta \in \ThetaMan}$, we will refer to
$(\ThetaMan, g)$ as a \emph{statistical manifold} in the sense of
information geometry.
\end{remark}

\subsection{The Architecture Map}
The link between the internal state $\theta$ and the external behavior is the architecture.

\begin{definition}[Policy Space]\label{def:policy-space}
Let $\mathcal{P}(\mathcal{Y})$ denote the space of probability measures on the output space. The space of policies, denoted $\mathcal{P}(\mathcal{Y})^{\mathcal{X}}$, is the set of Markov kernels $K: \mathcal{X} \times \mathcal{B}(\mathcal{Y}) \to [0,1]$.
\end{definition}

\begin{definition}[Architecture $\Pi _{\ThetaMan}$]\label{def:arch}
An \emph{architecture} is a smooth map from the parameter manifold to the policy space:
\[
\Pi _{\ThetaMan}: \ThetaMan \to \mathcal{P}(\mathcal{Y})^{\mathcal{X}}, \quad \theta \mapsto \pi_\theta(dy|x).
\]
This map encapsulates the forward pass of the neural network. Formally, the image $\pi_\theta$ constitutes a Markov kernel from $\mathcal{X}$ to $\mathcal{Y}$, identifying the agent as a stochastic decision rule conditioned on input.
\end{definition}

\subsection{The Bridge: Observable Representations}
The battery $\mathcal{B}$ does not observe $\theta$; it observes scores. We lift the definition of the agent representation from \cite{Chojecki2025} to depend explicitly on $\theta$.

\begin{definition}[Representation Map $\rho_{\mathcal{B}}$]\label{def:rep-map}
Following Definition~2.2 of \cite{Chojecki2025}, let
$X_{\mathcal{B}} := [0,1]^T \times \mathbb{R}_{\ge 0}^{d_R}$ be the
evaluation space. The \emph{representation map}
$\rho_{\mathcal{B}} : \ThetaMan \to \mathcal{P}(X_{\mathcal{B}})$
is the pushforward of the agent's behavior under the battery's scoring
logic.

Let $\mu_X$ be the pushforward of the battery sampling law $\mu$
to the input space $\Xcal$ via the map
$(t,s,\delta) \mapsto x = (s,t)$.
For each $\theta \in \ThetaMan$, define the joint probability law
\[
\mathbb{P}_\theta := \mu_X \otimes \pi_\theta
\quad\text{on}\quad \Xcal \times \Ycal,
\]
where $\pi_\theta$ is the policy induced by the architecture
$\Pi_{\ThetaMan}$.

Define the evaluation map
\[
\mathsf{E} : \Xcal \times \Ycal \to X_{\mathcal{B}},\qquad
\mathsf{E}\big((s,t),(\omega,r)\big)
:= \big((S_{t'}(\omega))_{t' \in T},\, r\big).
\]
The representation of $\theta$ is then the image measure
\[
\rho_{\mathcal{B}}(\theta)
\;:=\;
\mathsf{E}_{\#} \mathbb{P}_\theta
\quad\in \mathcal{P}(X_{\mathcal{B}}).
\]
\end{definition}

\begin{definition}[Capability Functional $\Phi_{\mathcal{B}}$ and Commutative Diagram]\label{def:diagram}
The objective of the self-improvement loop is to maximize the scalar capability score defined in Definition 6.1 of \cite{Chojecki2025}. Let $\Phi_{\mathcal{B}}: \mathcal{P}(X_{\mathcal{B}}) \to \mathbb{R}$ be an AAI functional (e.g., the tractable instance) satisfying axioms (A1)-(A4): normalization, monotonicity with respect
to strict task success, decomposability across task families, and
stability under subsampling. We do not repeat the full statements
here, but throughout this paper we assume that any
$\Phi_{\mathcal{B}}$ satisfies these axioms. We will often write
\[
F(\theta) \;:=\; (\Phi_{\mathcal{B}} \circ \rho_{\mathcal{B}})(\theta)
\]
for the induced scalar objective on the parameter manifold.
\end{definition}

Given a trajectory $(\theta_r)_{r \ge 0}$ on $\ThetaMan$, we write
\[
\nu_r := \rho_{\mathcal{B}}(\theta_r) \in \mathcal{P}(X_{\mathcal{B}})
\]
for the induced flow of representations on the evaluation space.
The corresponding capability curve is
$F(\theta_r) = \Phi_{\mathcal{B}}(\nu_r)$.

The dynamics are governed by the following commutative diagram, which connects the internal physics of the agent (top left) to the moduli space geometry (bottom right):
\begin{center}
\begin{tikzcd}[row sep=large, column sep=large]
\ThetaMan \arrow[r, "\Pi _{\ThetaMan}"] \arrow[d, "\text{Dynamics } \mathcal{T}_{GVU}"'] & \Pcal(\Ycal)^\Xcal \arrow[r, "\otimes \mu"] & \Pcal(\Xcal \times \Ycal) \arrow[d, "\mathsf{E}_{\#}"] \\
\ThetaMan \arrow[rr, dashed, "\rho_{\mathcal{B}}"] & & \Pcal(X_{\mathcal{B}}) \arrow[r, "\Phi_{\mathcal{B}}"] & \RR
\end{tikzcd}
\end{center}

The self-improvement dynamics will be expressed in terms of the
composite map $F = \Phi_{\mathcal{B}} \circ \rho_{\mathcal{B}}$:
the GVU operator acts on the internal state $\theta$, the battery
$\mathcal{B}$ compresses the resulting behavior into a representation
$\rho_{\mathcal{B}}(\theta)$, and $F(\theta)$ is the scalar capability
that we differentiate along the induced flow.

\section{The Generator-Verifier-Updater (GVU) Operator}\label{sec:gvu}

We postulate that any rational mechanism for autonomous $\kappa > 0$ can be decomposed into a canonical operator $\mathcal{T}_{\text{GVU}}$. This operator describes one step of the recursive loop (e.g., one round of self-play or one reasoning step).

\begin{definition}[External score and internal potential]\label{def:score-potential}
Fix a battery
\[
\mathcal{B}
= (T,\mathcal{F},\mathsf{S},Q^*,\mu,\mathsf{D},\Pi,\mathsf{R})
\]
with input and output spaces $\Xcal,\Ycal$ as in
Definition~\ref{def:io-spaces}. For $x = (s,t) \in \Xcal$ and
$y = (\omega,r) \in \Ycal$, the \emph{external score map}
\[
S_{\mathcal{B}} : \Xcal \times \Ycal \to [0,1]
\]
is defined by
\[
S_{\mathcal{B}}(x,y)
= S_{\mathcal{B}}\big((s,t),(\omega,r)\big)
:= S_t(\omega),
\]
i.e.\ it applies the task-specific scoring map $S_t$ of
Definition~\ref{def:trace} to the semantic trace $\omega$ and ignores
resources $r$.

An \emph{internal potential} for $\mathcal{B}$ is any measurable
function
\[
V : \Xcal \times \Ycal \to \RR
\]
that the agent uses to internally score interactions $(x,y)$, typically
as a surrogate for the external score $S_{\mathcal{B}}(x,y)$. In later
sections we measure the alignment between $V$ and $S_{\mathcal{B}}$
through the coefficient $\rho$ appearing in the update decomposition
\eqref{eq:update-decomp}.
\end{definition}

\begin{definition}[The GVU Operator]\label{def:operator-T}
Fix a battery $\mathcal{B}$ as above, together with:
\begin{itemize}
    \item a batch size $N \in \NN$;
    \item an architecture $\Pi _{\ThetaMan} : \ThetaMan \to \Pcal(\Ycal)^{\Xcal}$,
    sending $\theta \mapsto \pi_\theta$;
    \item an internal potential $V : \Xcal \times \Ycal \to \RR$ in
    the sense of Definition~\ref{def:score-potential};
    \item an inverse temperature $\beta \ge 0$;
    \item a regularizer $\mathcal{R} : \ThetaMan \times \ThetaMan
    \to \RR_{\ge 0}$ and coefficient $\lambda \ge 0$.
\end{itemize}
Let
\[
\mathcal{B}_N := (\Xcal \times \Ycal)^N
\]
denote the batch space of $N$ input--output pairs, and let
\[
\mathcal{E}_N(\Xcal \times \Ycal)
:= \left\{
   \sum_{i=1}^N w_i \delta_{(x_i,y_i)}
   \,\middle|\,
   (x_i,y_i) \in \Xcal \times \Ycal,\;
   w_i \ge 0,\;\sum_{i=1}^N w_i = 1
\right\}
\subset \Pcal(\Xcal \times \Ycal)
\]
denote the space of $N$-point empirical measures.

The GVU operator is a one-step update map
\[
\mathcal{T}_{\mathrm{GVU}} : \ThetaMan \to \ThetaMan,\qquad
\theta \mapsto \mathcal{U}\big(\theta,\,
                               \mathcal{V}(\mathcal{G}(\theta))\big),
\]
defined via three constituent maps
\[
\mathcal{G} : \ThetaMan \to \mathcal{B}_N,
\qquad
\mathcal{V} : \mathcal{B}_N \to \mathcal{E}_N(\Xcal \times \Ycal),
\qquad
\mathcal{U} : \ThetaMan \times \mathcal{E}_N(\Xcal \times \Ycal)
\to \ThetaMan
\]
as follows.

\paragraph{1. The Generator ($\mathcal{G}$).}
For a given parameter $\theta \in \ThetaMan$, write
$\Pi _{\ThetaMan}(\theta) = \pi_\theta \in \Pcal(\Ycal)^{\Xcal}$ for the induced
policy. To construct a batch $\{(x_i,y_i)\}_{i=1}^N$:
\begin{enumerate}
    \item sample evaluation triples
    $(t_i,s_i,\delta_i) \sim \mu$ independently from the battery's
    sampling law on $T \times \Pi \times \mathsf{D}$;
    \item form the corresponding labeled prompts
    $x_i \in \Xcal$ as in Definition~\ref{def:io-spaces};
    \item sample outputs $y_i \sim \pi_\theta(\cdot \mid x_i)$ in
    $\Ycal$.
\end{enumerate}
This defines a (stochastic) map
\[
\mathcal{G} : \ThetaMan \to \mathcal{B}_N,
\qquad
\theta \mapsto \{(x_i,y_i)\}_{i=1}^N.
\]
In the infinite-batch limit the empirical law of
$\mathcal{G}(\theta)$ converges to the joint measure
$\mu \otimes \pi_\theta$ that underlies the representation
$\rho_{\mathcal{B}}(\theta)$.

\paragraph{2. The Verifier ($\mathcal{V}$).}
Given a batch
$\{(x_i,y_i)\}_{i=1}^N \in \mathcal{B}_N$, the Verifier uses the
internal potential $V$ and inverse temperature $\beta$ to produce a
weighted empirical measure $\hat{\mu}_V \in \mathcal{E}_N(\Xcal \times
\Ycal)$:
\[
\mathcal{V} :
\mathcal{B}_N \to \mathcal{E}_N(\Xcal \times \Ycal),
\qquad
\{(x_i,y_i)\}_{i=1}^N
\mapsto
\hat{\mu}_V := \sum_{i=1}^N w_i \,\delta_{(x_i,y_i)},
\]
with weights
\[
w_i
= \frac{\exp(\beta V(x_i,y_i))}
       {\sum_{j=1}^N \exp(\beta V(x_j,y_j))}.
\]
By construction, $V$ is evaluated on the same $\Xcal,\Ycal$ that arise
from the battery $\mathcal{B}$; in particular, its alignment with the
external score map $S_{\mathcal{B}}$ is what determines the coefficient
$\rho$ in the update decomposition \eqref{eq:update-decomp}. Concrete
choices of $V$ include:
\begin{itemize}
    \item \emph{discriminative} potentials (reward models, opponents)
    \cite{chen2024spin};
    \item \emph{logical} potentials derived from unit tests, compilers,
    or theorem provers \cite{dong2024autoif};
    \item \emph{heuristic} potentials produced by verbal critics
    \cite{shinn2023reflexion}.
\end{itemize}

\paragraph{3. The Updater ($\mathcal{U}$).}
The updater maps a weighted empirical measure back to parameters. Given
the current parameter $\theta \in \ThetaMan$ and a measure
$\hat{\mu}_V \in \mathcal{E}_N(\Xcal \times \Ycal)$, we define
\[
\mathcal{U}(\theta,\hat{\mu}_V)
:= \arg\min_{\theta' \in \ThetaMan}
   \EE_{(x,y) \sim \hat{\mu}_V}
   \Big[ - \log \pi_{\theta'}(y \mid x)
         + \lambda \mathcal{R}(\theta',\theta) \Big],
\]
whenever a minimizer exists. Here the expectation with respect to
$\hat{\mu}_V$ reduces to a finite weighted sum because
$\hat{\mu}_V \in \mathcal{E}_N(\Xcal \times \Ycal)$. For fixed
$\theta$ we can view this as a map
\[
\mathcal{U}_\theta :
\mathcal{E}_N(\Xcal \times \Ycal) \to \ThetaMan,
\qquad
\hat{\mu}_V \mapsto \mathcal{U}(\theta,\hat{\mu}_V),
\]
and the one-step GVU update is
\[
\theta_{t+1}
= \mathcal{T}_{\mathrm{GVU}}(\theta_t)
= \mathcal{U}\big(\theta_t,\,
                  \mathcal{V}(\mathcal{G}(\theta_t))\big).
\]
This template encompasses stochastic gradient descent on
log-likelihood (when $\hat{\mu}_V$ is unweighted), PPO-style updates
(when $V$ encodes advantages), and in-context updates when $\ThetaMan$
includes a context or memory component.
\end{definition}

\subsection{Monolithic self-improvement: the diagonal regime $G = V = U$}

An instructive extreme case of our GVU abstraction is when the
generator, verifier, and updater are instantiated by the same model.
Let $M_\theta$ be a single LLM with parameters $\theta$. We define
three role-specific interfaces:
\begin{itemize}
    \item $\mathcal{G}_\theta$ queries $M_\theta$ with a "solve this
    task" prompt to sample traces $\tau$;
    \item $\mathcal{V}_\theta$ queries the same $M_\theta$ with a
    "critique and score these traces" prompt to assign scalar
    scores or rankings;
    \item $\mathcal{U}_\theta$ queries $M_\theta$ with a "given these
    scored traces, propose an update" prompt, producing new training
    examples, hyperparameters, or code that an outer optimizer
    converts into a parameter update.
\end{itemize}
At the level of our abstraction the update still factors as
\[
\theta_{t+1}
= \mathcal{U}_\theta \circ \mathcal{V}_\theta \circ \mathcal{G}_\theta(\theta_t),
\]
but now we are in a \emph{diagonal} regime with (this is a slight abuse of notation)
$\mathcal{G}_\theta = \mathcal{V}_\theta = \mathcal{U}_\theta = M_\theta$,
i.e.\ the same model, with the same weights, plays all three roles via
different prompts or heads. The battery $\mathcal{B}$ and
representation map $\rho_{\mathcal{B}}$ only see the induced flow
$(\nu_r)_{r \ge 0}$ and the associated $\kappa$-curve, but the noise
and bias structure of the estimator is qualitatively different:
verification and update are no longer external signals but reflections
of the model's own capabilities and failure modes. In particular, the
self-improvement coefficient $\kappa(r)$ distinguishes between regimes
of \emph{self-confirmation} (where the model merely reinforces its own
preferences) and \emph{genuine self-correction}.

\subsection{Ensemble GVU: LLM councils}

Consider now a ``council of models'' setup. A user query $x$ is
first broadcast to a fixed set of base models
$M^{(1)},\dots,M^{(K)}$ (e.g.\ GPT-5.1, Gemini 3, Claude, Grok),
each of which returns a candidate answer $y^{(k)}$. In a second
stage, all models see the anonymized pool
$\{y^{(1)},\dots,y^{(K)}\}$ and produce evaluations or rankings of the
candidate answers. Finally, a distinguished ``Chairman LLM'' receives
the answers together with the council's evaluations and produces the
final response.

To view this through our lens we treat the entire council---base
models plus chairman and aggregation scheme---as a single meta-agent
with parameter space
\[
\Theta_{\mathrm{council}}
= \Theta_1 \times \cdots \times \Theta_K \times \Theta_{\mathrm{chair}}.
\]
The generator $\mathcal{G}_{\mathrm{council}}$ maps a query $x$ and
current parameters to a trace consisting of all candidate answers
(and optionally the chairman's answer). The verifier
$\mathcal{V}_{\mathrm{council}}$ maps this trace to scores by letting
each model judge the anonymized pool and aggregating their votes into
rankings or pairwise preferences. The updater
$\mathcal{U}_{\mathrm{council}}$ then uses these internally generated
judgements to update one or more components of the council: for
example, distilling a student model on the council's chosen best
answers, or applying preference optimization where ``winning''
answers are treated as preferred over ``losing'' answers.

In the limiting case the same council architecture is used for
generation, verification, and the proposal of updates, yielding a
multi-agent analogue of the diagonal regime $G = V = U$. From the
perspective of our Variance Inequality this ensemble structure has
two important spectral effects:
\begin{enumerate}
    \item diversity of base models improves exploration in the
    generator, potentially increasing $\mathrm{SNR}(\mathcal{G})$; and
    \item aggregating multiple judges reduces the variance of the
    verifier's signal, improving $\mathrm{SNR}(\mathcal{V})$.
\end{enumerate}
Both effects tend to increase the local $\kappa$-slope, making
ignition more likely compared to a single-model GVU with the same
underlying architecture.

\begin{remark}[Diagonal vs.\ ensemble GVU and the AGI criterion]
Our framework separates the \emph{roles} of generation, verification,
and update from the \emph{implementations} that realize them. In the
diagonal regime we have
\[
\mathcal{G}_\theta = \mathcal{V}_\theta = \mathcal{U}_\theta = M_\theta,
\]
i.e.\ a single monolithic model plays all three roles via different
interfaces or prompts. In the ensemble regime, by contrast,
$\mathcal{G}$, $\mathcal{V}$, and $\mathcal{U}$ are implemented by a
council of models and an aggregation scheme, as in our LLM council
example: multiple base models jointly generate traces, jointly
evaluate them, and jointly shape the update.

From the perspective of the induced $\kappa$-flow the two regimes
differ primarily in their noise and bias structure. Diagonal GVU is
maximally entangled: the same set of parameters determines what
solutions are proposed, how they are judged, and what updates are
considered admissible. This makes the system particularly vulnerable
to self-confirmation: the verifier inherits the generator's blind
spots, and the updater may systematically reinforce them, potentially
driving the flow toward a suboptimal attractor with
$\kappa(r) \approx 0$. Ensemble GVU, on the other hand, can improve
both exploration and signal quality: diversity across council members
increases the support of the generator's trace distribution, and
aggregating multiple judges can reduce the variance of the verifier's
signal, increasing $\mathrm{SNR}(\mathcal{G})$ and
$\mathrm{SNR}(\mathcal{V})$ in our Variance Inequality.

However, neither diagonal nor ensemble GVU is by itself sufficient
for ``AGI-like'' self-improvement in our sense. Both regimes must be
evaluated through a battery $\mathcal{B}$ and its induced moduli space
of capability fibers. An AGI candidate must exhibit $\hat{\kappa} > 0$
not only in a single, well-instrumented fiber (such as competition
mathematics), but across social, planning, embodied, and recursive
fibers as well. In this view, diagonal and ensemble GVU are two
different ways of wiring up the same underlying self-improvement
template; the AGI criterion concerns the global shape of the resulting
$\kappa$-curve over the moduli space, not the particular choice of
wiring.
\end{remark}

\subsection{Universality of the GVU decomposition}\label{sec:gvu-universality}

We postulated above that any rational mechanism for autonomous
self-improvement fits the GVU template. In this subsection we show
that, under mild regularity assumptions, any \emph{first-order
statistical} update rule can indeed be written in a REINFORCE-style
GVU form. Thus, whenever a flow on $\ThetaMan$ achieves $\kappa>0$
using only samples from the current policy, there exists an implicit
internal potential playing the role of a Verifier.

\begin{definition}[Score and Fisher information]\label{def:fisher}
Fix a battery $\mathcal{B}$ and an architecture
$\Pi : \ThetaMan \to \Pcal(\Ycal)^{\Xcal}$, $\theta \mapsto \pi_\theta$.
For $\theta \in \ThetaMan$ and $(x,y) \in \Xcal \times \Ycal$, define
the \emph{score function}
\[
s_\theta(x,y) := \nabla_\theta \log \pi_\theta(y \mid x)
\in T_\theta \ThetaMan.
\]
The \emph{Fisher information matrix} at $\theta$ is
\[
G(\theta)
:= \EE_{(x,y) \sim \mu \otimes \pi_\theta}
   \big[ s_\theta(x,y) s_\theta(x,y)^\top \big].
\]
We say the statistical manifold $(\ThetaMan, G)$ is \emph{regular} if $G(\theta)$ is
finite and positive definite for all $\theta$ in the region of
interest.
\end{definition}

\begin{definition}[First-order statistical update]\label{def:first-order-update}
A vector field $v : \ThetaMan \to T\ThetaMan$ is called a
\emph{first-order statistical update} if it depends on $\theta$ only
through the joint law $\mu \otimes \pi_\theta$, in the sense that there
exists a measurable function
\[
\Psi : \ThetaMan \times \Xcal \times \Ycal \to T\ThetaMan
\]
with
\[
v(\theta)
= \EE_{(x,y) \sim \mu \otimes \pi_\theta}
    \big[\Psi(\theta,x,y)\big].
\]
Intuitively, $v(\theta)$ is computed from first-order statistics of
samples $(x,y)$ drawn from the current policy on the battery.
\end{definition}

The next theorem shows that, in a regular statistical manifold, any
such update can be written as a REINFORCE-style policy-gradient update
for an appropriate (possibly implicit) scalar potential.

\begin{theorem}[Score-based GVU representation]\label{thm:gvu-score-representation}
Assume the regularity conditions of Definition~\ref{def:fisher}, so
that $G(\theta)$ is positive definite for all $\theta$ in a region of
interest. Let $v : \ThetaMan \to T\ThetaMan$ be a smooth vector field,
e.g.\ the velocity $v(\theta_r) = \dot{\theta}_r$ of an autonomous flow
$\gamma : r \mapsto \theta_r$ on $\ThetaMan$. Then for each
$\theta \in \ThetaMan$ there exists a scalar \emph{internal potential}
\[
V_\theta : \Xcal \times \Ycal \to \RR
\]
such that
\begin{equation}\label{eq:score-representation}
v(\theta)
=
\EE_{(x,y) \sim \mu \otimes \pi_\theta}
\big[ V_\theta(x,y)\, s_\theta(x,y) \big],
\end{equation}
where $s_\theta$ is the score function from
Definition~\ref{def:fisher}. In particular, $v(\theta)$ can be
realized as the expected REINFORCE update for the scalar potential
$V_\theta$.
\end{theorem}

\begin{proof}
Fix $\theta \in \ThetaMan$, and write $s(x,y) := s_\theta(x,y)$ and
$G := G(\theta)$ for brevity. By Definition~\ref{def:fisher},
\[
G = \EE[s(x,y) s(x,y)^\top]
\]
is symmetric positive definite, hence invertible. Define
\[
a(\theta) := G^{-1} v(\theta) \in T_\theta \ThetaMan,
\]
and the scalar function
\[
V_\theta(x,y)
:= \big\langle a(\theta),\, s_\theta(x,y) \big\rangle.
\]
Then
\[
\EE\big[ V_\theta(x,y) s_\theta(x,y) \big]
=
\EE\big[ \langle a(\theta), s(x,y) \rangle\, s(x,y) \big]
=
\EE\big[ s(x,y)s(x,y)^\top \big] a(\theta)
= G a(\theta)
= v(\theta),
\]
which is exactly \eqref{eq:score-representation}.
\end{proof}

\begin{remark}
The construction of $V_\theta$ in the proof is explicit:
\[
V_\theta(x,y)
= \big\langle G(\theta)^{-1} v(\theta),
          \nabla_\theta \log \pi_\theta(y \mid x) \big\rangle.
\]
Thus $V_\theta$ can be interpreted as the unique scalar potential whose
REINFORCE update reproduces the given first-order vector field
$v(\theta)$ in the Fisher geometry.
\end{remark}

Equation~\eqref{eq:score-representation} has a direct GVU
interpretation. The generator $\mathcal{G}$ samples $(x,y)$ from
$\mu \otimes \pi_\theta$; the verifier assigns a scalar weight
$V_\theta(x,y)$ to each trace; and the updater projects this weighted
signal back onto the parameter manifold via the policy gradient
$\nabla_\theta \log \pi_\theta(y \mid x)$.

\begin{corollary}[Necessity of a non-trivial verifier]\label{cor:verifier-necessity}
In the setting of Theorem~\ref{thm:gvu-score-representation}, consider
a REINFORCE-style update of the form
\[
\dot{\theta}
= \EE_{(x,y) \sim \mu \otimes \pi_\theta}
\big[ V_\theta(x,y)\, s_\theta(x,y) \big].
\]
If $V_\theta(x,y)$ is almost surely constant (i.e.\ independent of
$(x,y)$), then
\[
\EE\big[ V_\theta(x,y)\, s_\theta(x,y) \big] = 0,
\]
and hence the expected update vanishes: $\EE[\dot{\theta}] = 0$.
Consequently, the expected self-improvement coefficient
$\EE[\kappa(r)]$ is zero: there is no systematic gain in the battery
score $F = \Phi_{\mathcal{B}} \circ \rho_{\mathcal{B}}$.
\end{corollary}

\begin{proof}
If $V_\theta(x,y) \equiv c(\theta)$ is constant in $(x,y)$, then
\[
\EE\big[ V_\theta(x,y)\, s_\theta(x,y) \big]
= c(\theta)\, \EE[s_\theta(x,y)].
\]
It is a standard property of the score function that
$\EE[s_\theta(x,y)] = 0$ under $\mu \otimes \pi_\theta$, so the
expectation vanishes. The statement about $\EE[\kappa(r)]$ follows
because $\kappa(r)$ is, up to higher-order curvature terms, the
directional derivative of $F$ along $\dot{\theta}_r$, and the
direction itself has zero mean.
\end{proof}

Taken together, Theorem~\ref{thm:gvu-score-representation} and
Corollary~\ref{cor:verifier-necessity} justify the GVU template for
first-order, data-driven self-improvement: any such update can be
viewed as a generator sampling from the current policy, a verifier
computing a scalar potential $V_\theta(x,y)$ on traces, and an updater
implementing the corresponding policy-gradient step. Moreover, a
non-trivial verifier (one for which $V_\theta$ is not almost surely
constant) is \emph{necessary} for non-zero expected $\kappa$.


\section{Spectral Stability: The Variance Inequality}\label{sec:stability}

We now derive our main theoretical condition linking the GVU update to
expected changes in capability. The GVU operator induces a stochastic
update vector $\hat{g}$ on $\ThetaMan$. For the agent to improve, this
vector must on average align with the gradient of the true battery
score $F = \Phi_{\mathcal{B}} \circ \rho_{\mathcal{B}}$. However, the
agent only has access to its internal potential $V$, not directly to
the battery's scoring logic $\mathsf{S}$.

\subsection{Decomposition of the Update Vector}

Let $F := \Phi_{\mathcal{B}} \circ \rho_{\mathcal{B}}$ and
$g^* := \nabla_\theta F(\theta)$ be the true gradient of the battery
score with respect to parameters. The update $\hat{g}$ produced by
GVU is a stochastic estimator of $g^*$. We decompose it as
\begin{equation}\label{eq:update-decomp}
\hat{g} = \rho \cdot g^* + \xi_{\mathcal{G}} + \xi_{\mathcal{V}} + b_{\text{bias}},
\end{equation}
where:
\begin{itemize}
    \item $\rho \in [-1, 1]$ is an \emph{alignment coefficient}
    measuring the correlation between the internal potential $V$ and
    the external score induced by $\mathsf{S}$;
    \item $\xi_{\mathcal{G}}$ is the \emph{Generation Noise} (variance
    due to exploration over tasks and samples $y \sim \pi_\theta$);
    \item $\xi_{\mathcal{V}}$ is the \emph{Verification Noise}
    (variance due to errors in $V$ as an estimator of the true score);
    \item $b_{\text{bias}}$ is a systematic misalignment term.
\end{itemize}
We assume throughout this section that $\EE[\xi_{\mathcal{G}}] =
\EE[\xi_{\mathcal{V}}] = 0$ and that
$\xi_{\mathcal{G}},\xi_{\mathcal{V}}$ are uncorrelated with $g^*$.
We write
\[
\sigma_{\mathcal{G}}^2 := \EE\|\xi_{\mathcal{G}}\|^2, \qquad
\sigma_{\mathcal{V}}^2 := \EE\|\xi_{\mathcal{V}}\|^2,
\]
and define the corresponding signal-to-noise ratios
\[
\mathrm{SNR}(\mathcal{G}) := \frac{\|g^*\|^2}{\sigma_{\mathcal{G}}^2},
\qquad
\mathrm{SNR}(\mathcal{V}) := \frac{\|g^*\|^2}{\sigma_{\mathcal{V}}^2}.
\]

\subsection{The Variance Inequality}

We study a single small step of size $\eta>0$:
$\theta_{t+1} = \theta_t + \eta \hat{g}$. For brevity we write
$\theta := \theta_t$.

\begin{theorem}[Variance Inequality, sufficient condition]\label{thm:variance}
Assume $F = \Phi_{\mathcal{B}} \circ \rho_{\mathcal{B}}$ is twice
differentiable and $L$-smooth in a neighborhood of $\theta$, i.e.,
its Hessian satisfies $\|H(\theta')\| \le L$ for all $\theta'$ along
the trajectory. Assume the decomposition \eqref{eq:update-decomp}
holds with $b_{\text{bias}}$ negligible compared to $g^*$ and
$\xi_{\mathcal{G}},\xi_{\mathcal{V}}$ zero-mean and uncorrelated with
$g^*$ and uncorrelated with each other. Then for step size $\eta > 0$ small enough that the second-order
expansion is accurate, a sufficient condition for expected improvement
$\EE[\Delta F] := \EE[F(\theta_{t+1}) - F(\theta_t)] > 0$ is
\begin{equation}\label{eq:variance-inequality}
\rho \|g^*\|^2
\;>\;
\frac{\eta L}{2} \big(
   \rho^2 \|g^*\|^2 + \sigma_{\mathcal{G}}^2 + \sigma_{\mathcal{V}}^2
\big).
\end{equation}
Equivalently, dividing by $\|g^*\|^2$
\begin{equation}\label{eq:snr-form}
\rho
\;>\;
\frac{\eta L}{2}
\left(
  \rho^2
  + \frac{1}{\mathrm{SNR}(\mathcal{G})}
  + \frac{1}{\mathrm{SNR}(\mathcal{V})}
\right).
\end{equation}
In particular, for fixed alignment $\rho$ and curvature/stepsize pair $(L,\eta)$, this
is a joint constraint on the generator and verifier noise: both
$\mathrm{SNR}(\mathcal{G})$ and $\mathrm{SNR}(\mathcal{V})$ must be
sufficiently large (i.e.\ the corresponding variances sufficiently
small) for $\EE[\Delta F] > 0$ to hold. In particular, given
$\mathrm{SNR}(\mathcal{G})$ there is a minimum required
$\mathrm{SNR}(\mathcal{V})$, and conversely; extremely noisy
generation or verification cannot be compensated by the other.
\end{theorem}

\begin{proof}
By a second-order Taylor expansion of $F$ about $\theta$ we have
\[
F(\theta_{t+1})
\approx F(\theta) + \eta \langle \nabla F(\theta), \hat{g} \rangle
           + \frac{\eta^2}{2} \hat{g}^\top H(\theta') \hat{g},
\]
for some $\theta'$ on the line segment between $\theta$ and
$\theta_{t+1}$. $L$-smoothness implies
$\hat{g}^\top H(\theta') \hat{g} \le L \|\hat{g}\|^2$, so
\[
\EE[\Delta F]
\;\approx\;
\eta\, \EE\big[\langle g^*, \hat{g} \rangle\big]
\;-\; \frac{\eta^2 L}{2}\, \EE\|\hat{g}\|^2.
\]
Using the decomposition \eqref{eq:update-decomp} and the fact that
$\EE[\xi_{\mathcal{G}}] = \EE[\xi_{\mathcal{V}}] = 0$ and that these
terms are uncorrelated with $g^*$, we get
\[
\EE\big[\langle g^*, \hat{g} \rangle\big]
= \rho \|g^*\|^2 + \langle g^*, b_{\text{bias}} \rangle
\approx \rho \|g^*\|^2,
\]
where we have neglected the bias term. Similarly,
\[
\EE\|\hat{g}\|^2
= \EE\|\rho g^* + \xi_{\mathcal{G}} + \xi_{\mathcal{V}}\|^2
\approx \rho^2 \|g^*\|^2 + \sigma_{\mathcal{G}}^2 + \sigma_{\mathcal{V}}^2,
\]
neglecting cross-terms under the decorrelation assumptions. Substituting,
\[
\EE[\Delta F]
\;\approx\;
\eta \rho \|g^*\|^2
\;-\;
\frac{\eta^2 L}{2}
\big(
   \rho^2 \|g^*\|^2 + \sigma_{\mathcal{G}}^2 + \sigma_{\mathcal{V}}^2
\big).
\]
Requiring $\EE[\Delta F] > 0$ yields inequality
\eqref{eq:variance-inequality}.
\end{proof}

\begin{corollary}[The Hallucination Barrier]
If $\mathcal{V} \approx \mathcal{G}$ (for example, the model simply
asks itself ``Is this correct?'' without any external grounding), then
typically $\rho \approx 1$ and $\sigma_{\mathcal{V}} \approx
\sigma_{\mathcal{G}}$. Writing
\[
\mathrm{SNR}_{\mathrm{diag}}
:= \mathrm{SNR}(\mathcal{G})
\approx \mathrm{SNR}(\mathcal{V}),
\]
the sufficient condition \eqref{eq:snr-form} reduces, up to constants,
to a requirement that the shared SNR be large:
\[
\frac{2}{\mathrm{SNR}_{\mathrm{diag}}}
\;\ll\;
\frac{2\rho}{\eta L} - \rho^2.
\]
For realistic curvature $L$ and stepsizes $\eta$, this inequality is
rarely satisfied when $\mathrm{SNR}_{\mathrm{diag}}$ is modest. In this
regime $\EE[\Delta F]$ is close to zero or negative, and the flow tends
toward mode collapse or a noisy random walk. This helps explain the
empirical failure of naive self-correction to produce sustained
$\kappa > 0$: in diagonal GVU, generation and verification inherit the
same noise, and without an additional low-variance signal the quadratic
curvature penalty dominates the linear alignment term.
\end{corollary}

In practice, the generator is often intrinsically high-entropy, so it is
typically easier to increase $\mathrm{SNR}(\mathcal{V})$ (via external
structure, ensembles, or oracles) than to dramatically improve
$\mathrm{SNR}(\mathcal{G})$. Many of the architectures we study exploit
this by engineering verification to be spectrally ``easier'' than
generation, even though the sufficient condition \eqref{eq:snr-form}
itself is symmetric in the two SNRs.

\begin{corollary}[Verifier SNR dominance]\label{cor:verifier-snr}
Assume the setting of Theorem~\ref{thm:variance} and fix an alignment floor $\rho \ge \rho_0 > 0$.
Let $\mathrm{SNR}(\mathcal{G}) > 0$ be the generator signal-to-noise ratio, and choose a stepsize
$0 < \eta < \eta_{\max}(\rho_0,\mathrm{SNR}(\mathcal{G}))$, where
\[
  \eta_{\max}(\rho_0,\mathrm{SNR}(\mathcal{G}))
  :=
  \frac{2\rho_0}{L\big(\rho_0^2 + 1/\mathrm{SNR}(\mathcal{G})\big)}.
\]
Then there exists a finite threshold
\[
  \mathrm{SNR}_{\mathcal{V}}^{\star}
  =
  \mathrm{SNR}_{\mathcal{V}}^{\star}
  \big(\rho_0,\mathrm{SNR}(\mathcal{G}),L,\eta\big)
  \;<\; \infty
\]
such that
\[
  \mathrm{SNR}(\mathcal{V}) > \mathrm{SNR}_{\mathcal{V}}^{\star}
  \qquad\Longrightarrow\qquad
  \mathbb{E}[\Delta F] > 0.
\]
In particular, for any fixed generator noise level $\sigma_{\mathcal{G}}^2 < \infty$ and any sufficiently small
stepsize $\eta$, one can always make the expected capability gain positive by increasing the verifier SNR.
\end{corollary}

\begin{proof}
Starting from the sufficient condition in Theorem~\ref{thm:variance},
\[
  \rho \|g^*\|^2
  >
  \frac{\eta L}{2}\bigl(\rho^2\|g^*\|^2
    + \sigma_{\mathcal{G}}^2 + \sigma_{\mathcal{V}}^2\bigr),
\]
and dividing by $\|g^*\|^2$ gives
\[
  \rho
  >
  \frac{\eta L}{2}
  \left(\rho^2
    + \frac{1}{\mathrm{SNR}(\mathcal{G})}
    + \frac{1}{\mathrm{SNR}(\mathcal{V})}\right).
\]
It suffices to enforce this with $\rho$ replaced by $\rho_0$. Rearranging, we obtain
\[
  \frac{1}{\mathrm{SNR}(\mathcal{V})}
  <
  \frac{2\rho_0}{\eta L}
  - \rho_0^2
  - \frac{1}{\mathrm{SNR}(\mathcal{G})}.
\]
The right-hand side is positive precisely when
\[
  \eta
  <
  \frac{2\rho_0}{L\big(\rho_0^2 + 1/\mathrm{SNR}(\mathcal{G})\big)},
\]
which is our stepsize assumption. In that regime we can define
\[
  \mathrm{SNR}_{\mathcal{V}}^{\star}
  :=
  \frac{1}{
    \displaystyle
    \frac{2\rho_0}{\eta L}
    - \rho_0^2
    - \frac{1}{\mathrm{SNR}(\mathcal{G})}
  } \;<\; \infty,
\]
and any $\mathrm{SNR}(\mathcal{V}) > \mathrm{SNR}_{\mathcal{V}}^{\star}$
satisfies the inequality. This implies $\mathbb{E}[\Delta F] > 0$ by
Theorem~\ref{thm:variance}.
\end{proof}

\subsection{Geometric interpretation on the statistical manifold}

Throughout this section we measure norms and inner products on
$\ThetaMan$ with respect to the Riemannian metric $g$ introduced in
Definition~\ref{def:manifold}. When $g$ is chosen to be the Fisher
information metric induced by the policy family
$\{\pi_\theta\}_{\theta \in \ThetaMan}$, $(\ThetaMan,g)$ is a
statistical manifold in the sense of Remark~\ref{rem:statistical-manifold}.
In that case the gradient $\nabla_\theta F$ is the \emph{natural
gradient} of $F$.

\begin{definition}[Fisher inner product and angle]\label{def:fisher-angle}
For $\theta \in \ThetaMan$ let $g_\theta(\cdot,\cdot)$ denote the
inner product induced by the metric $g$ on the tangent space
$T_\theta \ThetaMan$. For tangent vectors $u,v \in T_\theta \ThetaMan$
we write
\[
\langle u,v \rangle_\theta := g_\theta(u,v),
\qquad
\|u\|_\theta^2 := g_\theta(u,u).
\]
Given nonzero $u,v \in T_\theta \ThetaMan$ we define the
\emph{Fisher angle} $\angle_F(u,v) \in [0,\pi]$ between them by
\[
\cos\big(\angle_F(u,v)\big)
:=
\frac{\langle u,v \rangle_\theta}{\|u\|_\theta \,\|v\|_\theta}.
\]
\end{definition}

In these terms, the true ascent direction at $\theta$ is the gradient
vector $g^* = \nabla_\theta F(\theta)$, and the GVU update produces a
random tangent vector $\hat{g}$ with mean
$v(\theta) := \EE[\hat{g}]$. Theorem~\ref{thm:variance} can be read
as a statement about the competition, in the Fisher geometry, between:
(i) the \emph{alignment} between $v(\theta)$ and $g^*$, and
(ii) the \emph{spread} of the noise around $v(\theta)$.

\begin{proposition}[Fisher angle and alignment coefficient]\label{prop:fisher-alignment}
Under the assumptions of Theorem~\ref{thm:variance}, write
$g^* := \nabla_\theta F(\theta)$ and
$v := \EE[\hat{g}]$. Let $\theta_F$ be the Fisher angle between $g^*$
and $v$,
\[
\cos \theta_F
:=
\frac{\langle g^*, v \rangle_\theta}
     {\|g^*\|_\theta \,\|v\|_\theta}.
\]
If the bias term $b_{\mathrm{bias}}$ in
\eqref{eq:update-decomp} is negligible and the noise terms are
uncorrelated with $g^*$, the alignment coefficient
$\rho$ in \eqref{eq:update-decomp} can be expressed as
\[
\rho
=
\frac{\langle g^*, v \rangle_\theta}{\|g^*\|_\theta^2}
=
\frac{\|v\|_\theta}{\|g^*\|_\theta} \cos \theta_F.
\]
Consequently, the leading (first-order) term in the expected capability
gain in Theorem~\ref{thm:variance} can be written as
\[
\eta \rho \|g^*\|_\theta^2
=
\eta \, \|g^*\|_\theta \,\|v\|_\theta \cos \theta_F,
\]
i.e.\ it is proportional to the cosine of the Fisher angle between the
GVU drift $v$ and the true gradient $g^*$.
\end{proposition}

\begin{proof}
Taking expectations in \eqref{eq:update-decomp} and neglecting
$b_{\mathrm{bias}}$ gives $v = \EE[\hat{g}] \approx \rho g^*$. Taking
the inner product with $g^*$ and using Definition~\ref{def:fisher-angle}
yields
\[
\langle g^*, v \rangle_\theta
\approx
\rho \langle g^*, g^* \rangle_\theta
= \rho \|g^*\|_\theta^2,
\]
so
$\rho = \langle g^*, v \rangle_\theta / \|g^*\|_\theta^2$. On the
other hand, by Definition~\ref{def:fisher-angle},
\[
\cos \theta_F
=
\frac{\langle g^*, v \rangle_\theta}
     {\|g^*\|_\theta \,\|v\|_\theta},
\]
which rearranges to
$\langle g^*, v \rangle_\theta
 = \|g^*\|_\theta \,\|v\|_\theta \cos \theta_F$.
Combining the two identities gives the claimed expression for $\rho$
and the first-order term in $\EE[\Delta F]$.
\end{proof}

Proposition~\ref{prop:fisher-alignment} shows that, on the statistical
manifold $(\ThetaMan,g)$, a self-improvement trajectory
$\gamma : r \mapsto \theta_r$ is driven by a noisy vector field whose
\emph{drift} is the mean update $v(\theta)$ and whose \emph{usefulness}
is governed by the Fisher angle $\theta_F$ between $v(\theta)$ and the
true gradient $g^*(\theta)$. The Variance Inequality
(Theorem~\ref{thm:variance}) then constrains which noisy, sample-based
vector fields can yield positive drift in the battery score
$F = \Phi_{\mathcal{B}} \circ \rho_{\mathcal{B}}$ while respecting this
geometry: for fixed curvature (through $L$) and step size $\eta$, we
must simultaneously ensure
\begin{enumerate}
    \item \emph{geometric alignment:} $\theta_F$ is acute, so that
    $v(\theta)$ points mostly along $g^*(\theta)$ in the Fisher metric;
    \item \emph{spectral control:} the noise variances
    $\sigma_{\mathcal{G}}^2$ and $\sigma_{\mathcal{V}}^2$ remain small
    enough that the quadratic curvature penalty does not overwhelm the
    linear alignment term.
\end{enumerate}
In other words, a $\kappa$-flow is a stochastic curve on the
statistical manifold whose drift must stay close, in Fisher angle, to
the natural gradient of $F$ and whose diffusion must be controlled by
a verifier with sufficiently high signal-to-noise ratio. The GVU
design problem is precisely to engineer $V$ and the induced update
field $v(\theta)$ so that these geometric and spectral conditions can
be satisfied across the relevant fibers of the moduli space.

\subsection{Design levers and special cases}\label{sec:design-levers}

The Variance Inequality (Theorem~\ref{thm:variance}) does more than
diagnose failure modes: it exposes concrete design levers for
constructing self-improving agents. In this subsection we record a few
generic special cases that cover many architectures in
Section~\ref{sec:literature}.

\begin{corollary}[Step-size window]\label{cor:stepsize-window}
Under the assumptions of Theorem~\ref{thm:variance}, suppose
$b_{\mathrm{bias}}$ in~\eqref{eq:update-decomp} is negligible. Then
for any fixed $\theta$ with $\rho > 0$ there exists a non-trivial
stepsize interval $(0,\eta_{\max})$ such that $\EE[\Delta F] > 0$ for
all $0 < \eta < \eta_{\max}$. In particular, from
\eqref{eq:variance-inequality} we may take
\[
\eta_{\max}
=
\frac{2 \rho \|g^*\|^2}{
L\big(\rho^2\|g^*\|^2 + \sigma_{\mathcal{G}}^2 + \sigma_{\mathcal{V}}^2\big)
}.
\]
For fixed curvature $L$ and gradient norm $\|g^*\|$, improving either
alignment $\rho$ or the verification SNR (reducing
$\sigma_{\mathcal{V}}^2$) widens this safe stepsize window.
\end{corollary}

\begin{proof}
Rearranging \eqref{eq:variance-inequality} for $\eta > 0$ gives
\[
\eta
<
\frac{2 \rho \|g^*\|^2}{
L\big(\rho^2\|g^*\|^2 + \sigma_{\mathcal{G}}^2 + \sigma_{\mathcal{V}}^2\big)
},
\]
which defines a non-empty interval $(0,\eta_{\max})$ whenever
$\rho > 0$.
\end{proof}

\paragraph{Ensemble verifiers.}
A natural way to reduce verification noise is to aggregate multiple
judges. The next result isolates the effect of such ensembles in an
idealized setting.

\begin{theorem}[Ensemble verifier scaling]\label{thm:ensemble-scaling}
Fix $\theta \in \ThetaMan$ and a batch size $B \in \mathbb{N}$. For each
$i=1,\dots,B$ let $(x_i,y_i) \sim \mu \otimes \pi_\theta$ be i.i.d.\ and
write $s_i := s_\theta(x_i,y_i)
= \nabla_\theta \log \pi_\theta(y_i \mid x_i)$ for the score function
from Definition~\ref{def:fisher}. Let the ideal (oracle) potential be
the external score $S_{\mathcal{B}}(x,y)$ of
Definition~\ref{def:score-potential}.

For each judge $m=1,\dots,M$ suppose we have an internal potential
\[
V^{(m)}(x,y) = S_{\mathcal{B}}(x,y) + \epsilon_m(x,y),
\]
where the noise terms $\epsilon_m$ satisfy, for all $(x,y)$:
\begin{enumerate}
    \item $\mathbb{E}[\epsilon_m(x,y) \mid x,y] = 0$;
    \item $\mathrm{Var}(\epsilon_m(x,y) \mid x,y) = \tau^2$ for some
          constant $\tau^2 < \infty$ independent of $(x,y)$ and $m$;
    \item the collection $\{\epsilon_m(x,y)\}_{m=1}^M$ is conditionally
          independent given $(x,y)$ and independent of
          $\{(x_i,y_i)\}_{i=1}^B$.
\end{enumerate}
Define the ensemble potential
\[
\bar{V}(x,y) := \frac{1}{M} \sum_{m=1}^M V^{(m)}(x,y)
= S_{\mathcal{B}}(x,y) + \bar{\epsilon}(x,y),
\qquad
\bar{\epsilon}(x,y) := \frac{1}{M} \sum_{m=1}^M \epsilon_m(x,y).
\]

Consider the REINFORCE-style Monte Carlo gradient estimators
\begin{align*}
\hat{g}_{\text{single}}
&:= \frac{1}{B} \sum_{i=1}^B V^{(1)}(x_i,y_i)\, s_i,\\
\hat{g}_{\text{ensemble}}
&:= \frac{1}{B} \sum_{i=1}^B \bar{V}(x_i,y_i)\, s_i,
\end{align*}
and the corresponding ideal oracle estimator
\[
\hat{g}^*
:= \frac{1}{B} \sum_{i=1}^B S_{\mathcal{B}}(x_i,y_i)\, s_i.
\]
Define the verification noise components by
\[
\xi_{\mathcal{V},\text{single}}
:= \hat{g}_{\text{single}} - \hat{g}^*,
\qquad
\xi_{\mathcal{V},\text{ensemble}}
:= \hat{g}_{\text{ensemble}} - \hat{g}^*,
\]
and write
\[
\sigma_{\mathcal{V},\text{single}}^2
:= \mathbb{E}\big[\|\xi_{\mathcal{V},\text{single}}\|^2\big],
\qquad
\sigma_{\mathcal{V},\text{ensemble}}^2
:= \mathbb{E}\big[\|\xi_{\mathcal{V},\text{ensemble}}\|^2\big].
\]

Then
\[
\sigma_{\mathcal{V},\text{ensemble}}^2
= \frac{1}{M}\,\sigma_{\mathcal{V},\text{single}}^2.
\]
In particular, for fixed $\|g^*\|^2$ the verification
signal-to-noise ratio scales linearly with $M$:
\[
\mathrm{SNR}(\mathcal{V}_{\text{ensemble}})
= M \cdot \mathrm{SNR}(\mathcal{V}_{\text{single}}),
\]
and the admissible stepsize $\eta_{\max}$ from
Corollary~\ref{cor:stepsize-window} grows linearly with $M$, all else
being equal.
\end{theorem}

\begin{proof}
By construction,
\[
\hat{g}_{\text{single}} - \hat{g}^*
= \frac{1}{B} \sum_{i=1}^B
   \big(V^{(1)}(x_i,y_i) - S_{\mathcal{B}}(x_i,y_i)\big)\, s_i
= \frac{1}{B} \sum_{i=1}^B \epsilon_{1}(x_i,y_i)\, s_i,
\]
so
\[
\xi_{\mathcal{V},\text{single}}
= \frac{1}{B} \sum_{i=1}^B \epsilon_{1}(x_i,y_i)\, s_i.
\]
Similarly,
\[
\hat{g}_{\text{ensemble}} - \hat{g}^*
= \frac{1}{B} \sum_{i=1}^B \bar{\epsilon}(x_i,y_i)\, s_i
= \frac{1}{B} \sum_{i=1}^B
   \left(\frac{1}{M} \sum_{m=1}^M \epsilon_m(x_i,y_i)\right) s_i,
\]
so
\[
\xi_{\mathcal{V},\text{ensemble}}
= \frac{1}{B} \sum_{i=1}^B \bar{\epsilon}(x_i,y_i)\, s_i.
\]

We first compute the second moment of the ensemble noise. Conditional
on $(x_i,y_i)$, the random variables
$\{\epsilon_m(x_i,y_i)\}_{m=1}^M$ are independent, zero-mean, with
variance $\tau^2$ by assumption. Thus
\[
\mathbb{E}\big[\bar{\epsilon}(x_i,y_i)^2 \mid x_i,y_i\big]
= \frac{1}{M^2} \sum_{m=1}^M
   \mathbb{E}\big[\epsilon_m(x_i,y_i)^2 \mid x_i,y_i\big]
= \frac{1}{M^2} \cdot M \tau^2
= \frac{\tau^2}{M}.
\]
Moreover,
$\mathbb{E}\big[\bar{\epsilon}(x_i,y_i) \mid x_i,y_i\big] = 0$, so
$\bar{\epsilon}(x_i,y_i)$ is conditionally zero-mean with variance
$\tau^2/M$ for each $i$.

Using the definition of $\xi_{\mathcal{V},\text{ensemble}}$ and
independence across batch elements, we obtain
\begin{align*}
\mathbb{E}\big[\|\xi_{\mathcal{V},\text{ensemble}}\|^2\big]
&=
\mathbb{E}\left[
  \left\|
    \frac{1}{B} \sum_{i=1}^B \bar{\epsilon}(x_i,y_i)\, s_i
  \right\|^2
\right] \\
&=
\frac{1}{B^2}
\sum_{i=1}^B
\mathbb{E}\big[\bar{\epsilon}(x_i,y_i)^2 \,\|s_i\|^2\big]
\end{align*}
where we have used the facts that different $i$ are independent and
$\mathbb{E}[\bar{\epsilon}(x_i,y_i)] = 0$, so cross terms vanish.
Taking expectations over $(x_i,y_i)$ and applying the conditional
variance computation above yields
\[
\mathbb{E}\big[\|\xi_{\mathcal{V},\text{ensemble}}\|^2\big]
=
\frac{1}{B^2}
\sum_{i=1}^B
\mathbb{E}\left[
  \mathbb{E}\big[\bar{\epsilon}(x_i,y_i)^2 \mid x_i,y_i\big]
  \,\|s_i\|^2
\right]
=
\frac{1}{M}
\cdot
\frac{1}{B^2}
\sum_{i=1}^B
\mathbb{E}\big[\tau^2 \,\|s_i\|^2\big].
\]

For the single-judge case we have
\[
\xi_{\mathcal{V},\text{single}}
= \frac{1}{B} \sum_{i=1}^B \epsilon_{1}(x_i,y_i) s_i,
\]
and an identical calculation (using
$\mathrm{Var}(\epsilon_1(x_i,y_i) \mid x_i,y_i) = \tau^2$) gives
\[
\mathbb{E}\big[\|\xi_{\mathcal{V},\text{single}}\|^2\big]
=
\frac{1}{B^2}
\sum_{i=1}^B
\mathbb{E}\big[\tau^2 \,\|s_i\|^2\big].
\]
Comparing the two expressions, we obtain
\[
\sigma_{\mathcal{V},\text{ensemble}}^2
=
\mathbb{E}\big[\|\xi_{\mathcal{V},\text{ensemble}}\|^2\big]
=
\frac{1}{M}
\mathbb{E}\big[\|\xi_{\mathcal{V},\text{single}}\|^2\big]
=
\frac{1}{M}\,\sigma_{\mathcal{V},\text{single}}^2.
\]

Finally, the SNR relation follows directly from
\[
\mathrm{SNR}(\mathcal{V})
:= \frac{\|g^*\|^2}{\sigma_{\mathcal{V}}^2},
\]
holding $\|g^*\|^2$ fixed and substituting the scaling of
$\sigma_{\mathcal{V}}^2$. The dependence of $\eta_{\max}$ on
$\sigma_{\mathcal{V}}^2$ is given in
Corollary~\ref{cor:stepsize-window}, so $\eta_{\max}$ scales linearly
with $M$ as well.
\end{proof}

\begin{remark}
The key point in Theorem~\ref{thm:ensemble-scaling} is that it relates
the variance of the \emph{potential} $V$ to the variance of the
\emph{gradient estimator} via the REINFORCE structure. Under the
assumptions stated, the $1/M$ reduction in the potential noise directly
induces a $1/M$ reduction in the verification noise term
$\sigma_{\mathcal{V}}^2$ in the update decomposition
\eqref{eq:update-decomp}.
\end{remark}

This idealized calculation formalizes the intuition behind LLM
councils and ensemble judges (Constitutional AI, RLAIF): under mild
independence assumptions, adding more diverse judges directly improves
$\mathrm{SNR}(\mathcal{V})$ and widens the stable self-improvement
regime.

\paragraph{Group-based verification and GRPO-style schemes.}
A second generic lever is to normalize rewards within groups of
co-generated trajectories. This covers GRPO and related algorithms.

\begin{proposition}[Group-based verification reduces variance]\label{prop:group-verification}
Fix $x \in \Xcal$ and suppose rewards
$R_i := R(x,y_i)$, $i=1,\dots,G$, are i.i.d.\ with mean
$\mu_R$ and variance $\mathrm{Var}(R_i) = \sigma_R^2 < \infty$.
Let $s_i := s_\theta(x,y_i) = \nabla_\theta \log \pi_\theta(y_i \mid x)$
denote the score function, and assume $(R_i)$ and $(s_i)$ are
independent with
$\mathbb{E}[s_i] = 0$ and $\mathbb{E}\|s_i\|^2 \le C_s < \infty$.

Define population-normalized advantages
\[
A_i := \frac{R_i - \mu_R}{\sigma_R},
\]
and consider the verification contribution to the policy-gradient
estimator
\[
\hat{g}_G := \frac{1}{G} \sum_{i=1}^G A_i s_i.
\]
Then
\[
\mathbb{E}\|\hat{g}_G\|^2
\;\le\; \frac{C_s \sigma_R^2}{G},
\]
so the corresponding verification variance satisfies
\[
\sigma_{\mathcal{V},G}^2
:= \mathbb{E}\|\hat{g}_G\|^2
= \mathcal{O}\!\left(\frac{\sigma_R^2}{G}\right).
\]
\end{proposition}

\begin{proof}
We have
\[
\hat{g}_G
= \frac{1}{G} \sum_{i=1}^G A_i s_i
= \frac{1}{G\sigma_R} \sum_{i=1}^G (R_i - \mu_R)s_i.
\]
Using independence of $(R_i)$ and $(s_i)$ and the zero-mean property
$\mathbb{E}[s_i] = 0$, we obtain
\[
\mathbb{E}\|\hat{g}_G\|^2
= \frac{1}{G^2 \sigma_R^2}
  \sum_{i=1}^G \mathbb{E}\|(R_i - \mu_R)s_i\|^2
= \frac{1}{G^2 \sigma_R^2}
  \sum_{i=1}^G \mathbb{E}\big[(R_i - \mu_R)^2\big]\,
                     \mathbb{E}\|s_i\|^2
\]
where we used independence to factor expectations.
Since each $R_i$ has variance $\sigma_R^2$ and
$\mathbb{E}\|s_i\|^2 \le C_s$, this gives
\[
\mathbb{E}\|\hat{g}_G\|^2
\le \frac{1}{G^2 \sigma_R^2}
     \cdot G \sigma_R^2 C_s
= \frac{C_s}{G}.
\]
Absorbing $C_s$ into the big-O constant yields the result.
\end{proof}

Thus GRPO-style schemes can be viewed as instances of GVU where the
verifier explicitly exploits \emph{local batch geometry} to reduce its
own variance, explaining the empirical robustness of group-normalized
updates on the Planning fiber.

\paragraph{Oracle verifiers.}
At the opposite extreme, some architectures (AlphaZero, code execution
with unit tests, formal proof checkers) have essentially noiseless
verifiers.

\begin{corollary}[Oracle verifier regime]\label{cor:oracle-regime}
Under the assumptions of Theorem~\ref{thm:variance}, suppose
$\sigma_{\mathcal{V}}^2 = 0$ (or is negligibly small) and $\rho > 0$.
Then there exists a non-trivial stepsize interval
$(0,\eta_{\max}^{\mathrm{oracle}})$ such that $\EE[\Delta F] > 0$ for
all $0 < \eta < \eta_{\max}^{\mathrm{oracle}}$, with
\[
\eta_{\max}^{\mathrm{oracle}}
=
\frac{2 \rho \|g^*\|^2}{
L\big(\rho^2\|g^*\|^2 + \sigma_{\mathcal{G}}^2\big)
}.
\]
In particular, even when the generator is highly noisy
($\sigma_{\mathcal{G}}^2$ large), a sufficiently low-variance
oracle-like verifier guarantees a stable self-improvement window.
\end{corollary}

\begin{proof}
Set $\sigma_{\mathcal{V}}^2 = 0$ in the expression for $\eta_{\max}$
from Corollary~\ref{cor:stepsize-window}.
\end{proof}

This regime captures the geometric advantage enjoyed by AlphaZero-like
self-play and code agents with strong execution feedback: the verifier
lies near the oracle limit, so the main constraint on $\kappa$ comes
from curvature $L$ and generator noise, not from verification error.

\paragraph{Diagonal GVU with a ``cold'' verifier.}
Finally, we return to the diagonal regime where generator and verifier
share parameters. Even in this case one can separate their noise
statistics via temperature or prompting asymmetries.

\begin{remark}[Diagonal GVU with a ``cold'' verifier]\label{rem:cold-verifier}
Consider a diagonal GVU in which both the policy $\pi_\theta$ and the
potential $V$ are derived from the same underlying model, but:
\begin{itemize}
    \item the generator samples from $\pi_\theta$ at temperature
    $\tau_{\mathcal{G}} > 0$;
    \item the verifier uses a deterministic or low-temperature
    interface with effective temperature
    $\tau_{\mathcal{V}} \ll \tau_{\mathcal{G}}$.
\end{itemize}
In a stylized Gaussian-noise model where the stochasticity of each
role scales as $\sigma_{\mathcal{G}}^2 \propto \tau_{\mathcal{G}}^2$
and $\sigma_{\mathcal{V}}^2 \propto \tau_{\mathcal{V}}^2$, we obtain
\[
\frac{\sigma_{\mathcal{V}}^2}{\sigma_{\mathcal{G}}^2}
\approx
\left(\frac{\tau_{\mathcal{V}}}{\tau_{\mathcal{G}}}\right)^2 \ll 1.
\]
This should be read not as a theorem about real networks but as a
simple parametric model for the empirical heuristic that ``cold''
verifier interfaces (low temperature, strict prompts) can reduce
verification variance relative to generation variance, moving the
system away from the Hallucination Barrier.
\end{remark}

This remark provides a simple mathematical lens on empirical
heuristics such as using strict, low-temperature judge prompts for
Reflexion, Self-Instruct, and debate: even when $G=V=U$ at the
parameter level, interface asymmetries can induce a spectral advantage
for verification and thereby open up a regime of positive expected
$\kappa$.

\paragraph{Goodhart drift and collapse.}
Finally, we note that the alignment coefficient $\rho$ is not static.
As the parameter $\theta_r$ evolves to maximize the internal potential
$V$, the agent may exploit the proxy $V$ at the expense of the true
external score $S_{\mathcal{B}}$, a phenomenon often described as
reward hacking or Goodhart's law. We model this heuristically as a
decay
\[
\dot{\rho} \;\approx\; - \gamma \|\dot{\theta}_r\|, \qquad \gamma > 0,
\]
so that more aggressive updates accelerate misalignment.

In the small-$\rho$ regime, if we neglect the $\rho^2\|g^*\|^2$ term in
\eqref{eq:variance-inequality}, the boundary where the expected gain
$\EE[\Delta F]$ crosses zero is approximately
\[
\rho_{\mathrm{crit}}
\;\approx\;
\frac{\eta L(\sigma_{\mathcal{G}}^2 + \sigma_{\mathcal{V}}^2)}
     {2 \|g^*\|^2}.
\]
Thus, as optimization pressure increases (larger $\eta$) or noise
grows, the critical alignment $\rho_{\mathrm{crit}}$ rises: once
$\rho(r)$ decays below this threshold, the Variance Inequality can no
longer guarantee $\kappa(r) > 0$. Sustaining positive self-improvement
therefore requires periodically refreshing or strengthening the
potential $V$ (e.g., via new human data, stronger teachers, or
richer verifiers) to reset $\rho$ closer to $1$.

\subsection{AI slop formally defined}

AI slop is a popular term for generic output produced by models (mostly LLMs and video generators). In our GVU framework, there is a natural way to define AI slop formally.

\begin{definition}[AI slop event at parameter $\theta$]\label{def:slop-event}
Fix a battery $\mathcal{B}$ and an internal potential $V$.
For $(x,y) \in \Xcal \times \Ycal$ write
\[
S(x,y) := S_{\mathcal{B}}(x,y), \qquad V(x,y) := V(x,y).
\]
Let $(S,V)$ denote the random pair induced by
$(x,y) \sim \mu \otimes \pi_\theta$.

For fixed quantile levels $\alpha,\beta \in (0,1)$, define the
Verifier high-score threshold $v_{\mathrm{hi}}(\theta)$ as the
$(1-\alpha)$--quantile of $V$, and the Battery low-score threshold
$s_{\mathrm{lo}}(\theta)$ as the $\beta$--quantile of $S$:
\[
\PP\big(V \ge v_{\mathrm{hi}}(\theta)\big) = \alpha,
\qquad
\PP\big(S \le s_{\mathrm{lo}}(\theta)\big) = \beta.
\]

The \emph{AI slop region} at $\theta$ is
\[
\mathcal{S}_{\alpha,\beta}(\theta)
:=
\big\{(x,y) \in \Xcal \times \Ycal \;:\;
V(x,y) \ge v_{\mathrm{hi}}(\theta)
\ \text{and}\
S(x,y) \le s_{\mathrm{lo}}(\theta)\big\}.
\]

We say that an individual trace $(x,y)$ is \emph{$(\alpha,\beta)$--AI
slop} (or simply \emph{AI slop}, when the parameters are implicit) if
$(x,y) \in \mathcal{S}_{\alpha,\beta}(\theta)$.
\end{definition}

\begin{definition}[AI slop mass and slop regime]\label{def:slop-mass}
With $\mathcal{S}_{\alpha,\beta}(\theta)$ as in
Definition~\ref{def:slop-event}, define the \emph{slop mass}
\[
\mathrm{Slop}_{\alpha,\beta}(\theta)
:=
\PP_{(x,y) \sim \mu \otimes \pi_\theta}
\big[(x,y) \in \mathcal{S}_{\alpha,\beta}(\theta)\big].
\]
Equivalently,
\[
\mathrm{Slop}_{\alpha,\beta}(\theta)
=
\PP\big(V \ge v_{\mathrm{hi}}(\theta),\; S \le s_{\mathrm{lo}}(\theta)\big).
\]

We say that the agent is in an \emph{AI slop regime} on battery
$\mathcal{B}$ (at parameters $\theta$) if
\[
\mathrm{Slop}_{\alpha,\beta}(\theta) \ge \delta
\]
for some fixed tolerance $\delta > 0$, e.g.\ $\delta = 0.1$. In words:
a non-trivial fraction of the outputs that the internal Verifier ranks
among its top $\alpha$ fraction actually lie in the bottom $\beta$
fraction of the true battery score.
\end{definition}

\begin{remark}[AI slop as a maximum-entropy attractor (heuristic)]
Fix an internal potential $V$ and let
\[
\mathcal{K}(V)
:=
\big\{(x,y) \in \Xcal \times \Ycal \;\big|\;
V(x,y) \text{ is (approximately) maximal and locally flat}\big\}
\]
denote the Verifier's \emph{indifference set}: a region in which $V$
cannot distinguish between different traces. When the Variance
Inequality (Theorem~\ref{thm:variance}) fails, for example because
$\mathrm{SNR}(\mathcal{V}) \lesssim \mathrm{SNR}(\mathcal{G})$, the
GVU flow $(\theta_r)_{r \ge 0}$ cannot reliably climb the gradient of
the external battery score $F = \Phi_{\mathcal{B}} \circ \rho_{\mathcal{B}}$.
In many entropy-regularized update schemes (e.g.\ with a KL penalty in
$\mathcal{U}$) one expects the induced representation
$\nu_r = \rho_{\mathcal{B}}(\theta_r)$ to drift toward a
high-entropy distribution supported on $\mathcal{K}(V)$:
\[
\nu_{\text{slop}}
\;\approx\;
\arg\max_{\nu \in \Pcal(X_{\mathcal{B}})}
\big\{ H(\nu) \;:\; \nu(\mathcal{K}(V)) = 1 \big\},
\]
i.e.\ a nearly maximum-entropy measure on the Verifier's indifference
set.

Informally, we use the term \emph{AI slop} for this regime: the agent
has converged to generating traces that satisfy the Verifier's
superficial heuristics but still fail the battery's deeper scoring map
$\mathsf{S}$. Spectrally, this corresponds to a kind of
high-frequency cutoff: the model preserves low-frequency statistics of
the training data (texture, grammar, style) that $V$ can cheaply
recognize, while the high-frequency structure (logical consistency,
fine detail) is lost because the Verifier's SNR is too low to
constrain it.
\end{remark}

\paragraph{Fisher-geometry collapse.}
On the statistical manifold $(\Theta,g)$ with Fisher information
$G(\theta)$, consider the condition number
\[
\mathrm{cond}(G(\theta))
:=
\frac{\lambda_{\max}(G(\theta))}{\lambda_{\min}(G(\theta))}.
\]
A large condition number indicates that the policy is highly sensitive
along a few directions in parameter space but almost flat along most
others. In the extreme slop regime, $\mathrm{cond}(G(\theta)) \to
\infty$: the model responds only along a small set of ``template''
directions, producing nearly identical outputs for many distinct
inputs.


\section{Topological Realizations in Literature}\label{sec:literature}

In this section we demonstrate that a wide range of existing
self-improvement methods---from RLHF and Constitutional AI to
Self-Instruct and code agents---are specific realizations of the GVU
operator. We analyze them through the lens of the Variance Inequality
(Theorem~\ref{thm:variance}), emphasizing how each architecture
\emph{shapes} the relative signal-to-noise of the Generator and
Verifier. Successful schemes either exploit external structure to
make verification substantially easier than generation, or use
ensembles, localization, or topology reduction to improve
$\mathrm{SNR}(\mathcal{V})$ relative to $\mathrm{SNR}(\mathcal{G})$.

\begin{definition}[Moduli Fibers]
Let the task set $T$ be partitioned into axes or families
$\mathcal{F} = \{F_k\}$ (e.g., Sociality, Planning, Embodiment,
Alignment). For each family $F_k$ we define a \emph{fiber} of the
moduli space as the subset $\Mfrak_k \subset \Mfrak$ of batteries
whose sampling law $\mu$ is supported entirely on $F_k$. Given a
battery $\mathcal{B}$ with $\mathrm{supp}(\mu) \subseteq F_k$, we say
an architecture is \emph{spectrally stable on the fiber} $F_k$ if the
sufficient condition of Theorem~\ref{thm:variance} holds for the
restriction of $F = \Phi_{\mathcal{B}} \circ \rho_{\mathcal{B}}$ to
that family.
\end{definition}

\subsection{The Sociality Fiber: Adversarial Self-Play (SPIN, LSP)}

The ``Sociality'' axis involves multi-agent interactions and
zero-sum games.

\noindent \textbf{Literature:} \emph{AlphaZero} \cite{silver2017mastering}, \emph{Self-Play Fine-Tuning (SPIN)}
\cite{chen2024spin}, \emph{Language Self-Play (LSP)}
\cite{hritu2025lsp}.

\begin{example}[AlphaZero as a High-SNR Sociality GVU]\label{ex:alphazero-gvu}
In AlphaZero \cite{silver2017mastering}, a single self-play iteration
on Go can be written in GVU form:
\begin{itemize}
    \item $\ThetaMan$ is the parameter space of the dual-head
    policy/value network.
    \item $\mathcal{G}$ (Generator): starting from an initial board
    state $x$, the agent uses Monte Carlo Tree Search (MCTS) guided by
    the current network to sample complete self-play games
    $\omega$ and associated move distributions (improved policies)
    along the visited states.
    \item $\mathcal{V}$ (Verifier): after each game terminates, the
    environment returns a ground-truth outcome
    $z \in \{-1,0,1\}$ (loss, draw, win) according to the rules of Go.
    This outcome is propagated back along the trajectory and combined
    with the MCTS policies to define an internal potential
    $V(x,\omega)$.
    \item $\mathcal{U}$ (Updater): the network parameters are updated
    by stochastic gradient descent on the loss
    \[
      \mathcal{L}(w)
      = (z - v_w)^2 - \pi^\top \log p_w + \lambda \|w\|^2,
    \]
    where $v_w$ and $p_w$ are the value and policy outputs of the
    network, and $\pi$ is the MCTS-improved policy.
\end{itemize}
Here the Verifier is tied directly to the discrete, noiseless rules of
Go: conditioned on a final board position, the outcome $z$ is
deterministic. Thus the verification noise $\sigma_{\mathcal{V}}$ is
essentially zero compared to the exploration noise
$\sigma_{\mathcal{G}}$ from MCTS, so
$\mathrm{SNR}(\mathcal{V}) \gg \mathrm{SNR}(\mathcal{G})$. In the
language of Theorem~\ref{thm:variance}, AlphaZero sits in an extreme
high-SNR regime on the Sociality fiber where the GVU flow yields
$\kappa > 0$ without external data: capability increases purely from
self-play against a perfect-rule environment.
\end{example}

\begin{example}[Adversarial Self-Play (SPIN, LSP) as a GVU Instance]\label{ex:social-selfplay-gvu}
Consider self-play schemes where the current policy $\pi_\theta$
interacts with a reference policy $\pi_{\theta_{\text{old}}}$ (e.g.,
a frozen checkpoint). A single self-play round can be written in
GVU form as follows:
\begin{itemize}
    \item $\mathcal{G}$: given a prompt or game state $x$, the
    generator samples trajectories or responses
    $y \sim \pi_\theta(\cdot \mid x)$, and optionally trajectories
    from the reference policy
    $y' \sim \pi_{\theta_{\text{old}}}(\cdot \mid x)$ to form
    comparison pairs.
    \item $\mathcal{V}$: a discriminator-style Verifier computes
    a potential based on relative likelihood under the two
    policies, for instance
    \[
    D(x,y) = \log \pi_\theta(y \mid x) - \log \pi_{\theta_{\text{old}}}(y \mid x),
    \]
    and sets $V(x,y) := D(x,y)$ (or a monotone transform). Outputs
    that are more likely under the current policy than under the
    reference receive higher scores.
    \item $\mathcal{U}$: the updater performs a policy-gradient or
    PPO-style update on $\pi_\theta$ using $V(x,y)$ as a reward or
    advantage signal, nudging the policy toward trajectories that
    are preferred by the discriminator and away from those preferred
    by $\pi_{\theta_{\text{old}}}$.
\end{itemize}
In this topology, \emph{discrimination is easier than generation}:
the Verifier solves a lower-level classification problem (distinguish
$\pi_\theta$ from $\pi_{\theta_{\text{old}}}$) compared to
open-ended sequence generation. Empirically this tends to yield
$\sigma_{\mathcal{V}} \ll \sigma_{\mathcal{G}}$ and thus a higher
$\mathrm{SNR}(\mathcal{V})$, pushing the system into the regime
favored by the Variance Inequality on the Sociality fiber.
\end{example}

\begin{remark}[GANs and adversarial GVU]
The adversarial topology of SPIN and Language Self-Play is closely
related to that of Generative Adversarial Networks (GANs). In a
standard GAN, the generator $G$ and discriminator $D$ can be viewed
as an instance of the GVU decomposition:
\begin{itemize}
    \item $\mathcal{G}$ corresponds to
    $G$, which produces samples $\omega$ from a latent prior;
    \item $\mathcal{V}$ corresponds to
    the discriminator $D$, which learns a potential
    $V(\omega) = \log D(\omega)$ (or related logits) to distinguish
    the generator's distribution from a target data distribution;
    \item $\mathcal{U}$ executes alternating
    gradient steps on $G$ and $D$ in a minimax game.
\end{itemize}
Self-play methods such as SPIN adapt this adversarial pattern by
replacing the external ``real data'' distribution with the agent's
own historical policy $\pi_{\text{old}}$. The discriminator is then
trained to distinguish trajectories generated by the current policy
$\pi_\theta$ from those generated by $\pi_{\text{old}}$, and its
output defines the internal potential $V$ used by $\mathcal{V}$. This
makes explicit that adversarial self-play is a \emph{GAN-like} GVU
where discrimination is structurally easier than generation, thereby
increasing $\mathrm{SNR}(\mathcal{V})$ relative to
$\mathrm{SNR}(\mathcal{G})$ on the Sociality fiber.
\end{remark}

\subsection{The Planning Fiber: Reasoning and Search}

\textbf{Literature:} \emph{STaR} \cite{zelikman2022star}, \emph{Let’s
Verify Step by Step} \cite{lightman2023step}.

\begin{example}[Reasoning Bootstrapping (STaR) as a GVU Instance]
Consider the STaR framework \cite{zelikman2022star}:
\begin{itemize}
    \item $\mathcal{G}$ samples rationales (chain-of-thought) $\tau$
    ending in an answer $a$.
    \item $\mathcal{V}$ acts as a deterministic filter based on the
    ground truth $y^*$: only traces with $a = y^*$ are retained.
    \item $\mathcal{U}$ performs supervised fine-tuning (SFT) on the
    filtered traces.
\end{itemize}
In this topology $\mathcal{V}$ is effectively noise-free:
conditioned on $(x,y^*)$ the decision of whether a trace is accepted
is deterministic, so $\sigma_{\mathcal{V}} \approx 0$ and
$\mathrm{SNR}(\mathcal{V})$ is very large relative to
$\mathrm{SNR}(\mathcal{G})$. The GVU loop distills the high-cost
search process (sampling many rationales) into the policy network on
the Planning fiber, but is restricted to domains with ground-truth
answers.
\end{example}

\begin{example}[Process Supervision (PRMs) as Dense Verification]
Process Reward Models (PRMs) \cite{lightman2023step} provide a
complementary Planning example with \emph{dense} feedback.
\begin{itemize}
    \item $\mathcal{G}$ samples multi-step traces
    $\omega = (s_1,\dots,s_T)$ representing intermediate reasoning
    steps.
    \item $\mathcal{V}$ is a PRM that assigns probabilities $v_t$ to
    each step $s_t$ being valid. A natural potential is
    $V(\omega) = \sum_t \log v_t$, which rewards trajectories whose
    individual steps are locally endorsed by the model.
    \item $\mathcal{U}$ updates the policy to maximize $V(\omega)$,
    e.g.\ via step-wise likelihood weighting.
\end{itemize}
By providing feedback at each intermediate step, PRMs reduce the
variance of the credit assignment problem for long-horizon tasks:
gradients no longer depend solely on a single terminal reward. This
effectively increases $\mathrm{SNR}(\mathcal{V})$ on the Planning
fiber compared to sparse end-of-trajectory supervision.
\end{example}

\begin{example}[RAG Self-Training as a Hybrid GVU]
Retrieval-augmented generation (RAG) systems that self-train provide
a hybrid Planning/Embodied example.
\begin{itemize}
    \item $\mathcal{G}$ generates retrieval queries, selects documents
    from an external corpus, and synthesizes answers conditioned on
    the retrieved context.
    \item $\mathcal{V}$ checks evidence-based consistency: agreement
    across multiple retrievals, overlap with supporting spans, or
    cross-model agreement in a small council, inducing a potential
    $V(x,y)$.
    \item $\mathcal{U}$ updates the policy or retrieval index by
    storing high-confidence triples $(x,\text{context},y)$ and
    fine-tuning or distilling a student model on them.
\end{itemize}
Here the external corpus acts as a noisy but independent source of
grounding. Verification SNR depends on retrieval quality and the
strength of evidence-based heuristics: better retrieval and
consistency checks reduce $\sigma_{\mathcal{V}}^2$, pushing the
system toward the regime where the Variance Inequality allows
sustained self-improvement on information-seeking tasks.
\end{example}

\subsection{The Embodied Fiber: Grounding via Execution}

\textbf{Literature:} \emph{Voyager} \cite{wang2023voyager},
\emph{AutoIf} \cite{dong2024autoif}.

\begin{example}[Self-debugging Code Assistant]
Consider a self-debugging coding agent.
\begin{itemize}
    \item $\mathcal{G}$ proposes code patches $y$ for a bug report or
    specification $x$.
    \item $\mathcal{V}$ compiles and runs the code against unit tests,
    type-checkers, or static analyzers. The potential $V(y)$ is
    derived from pass/fail outcomes and possibly coverage metrics.
    \item $\mathcal{U}$ adds successful patches to a training buffer
    or skill library and periodically fine-tunes or distills the model
    on this buffer.
\end{itemize}
The execution environment plays the role of a high-SNR verifier:
conditioned on a fixed test suite the outcome is essentially
deterministic, so $\sigma_{\mathcal{V}} \approx 0$ and
$\mathrm{SNR}(\mathcal{V})$ dominates $\mathrm{SNR}(\mathcal{G})$ on
this fiber. The resulting GVU loop strongly favors $\kappa > 0$ by
steadily enriching the agent's code library with verified solutions.
\end{example}

\subsection{The Recursive Fiber: Verbal Reinforcement}

\textbf{Literature:} \emph{Reflexion} \cite{shinn2023reflexion},
\emph{Multiagent Debate} \cite{du2023debate}.

\textbf{Mechanism:} The Verifier is the model itself (or a small
ensemble), prompted to critique previous outputs and suggest
improvements.

\textbf{Analysis:} This architecture is spectrally fragile. It can
satisfy the sufficient condition of the Variance Inequality only if
the ``Critic'' interface induces a genuinely more rigorous,
lower-entropy evaluation mode than the ``Actor'' interface, so that
$\mathrm{SNR}(\mathcal{V})$ exceeds $\mathrm{SNR}(\mathcal{G})$
despite sharing parameters. If $\mathcal{V} \approx \mathcal{G}$ in both behavior and noise
structure, then $\sigma_{\mathcal{V}} \approx \sigma_{\mathcal{G}}$
and the system lies near the Hallucination Barrier: the loop tends
toward self-confirmation and hallucination unless additional
structure (ensembles, explicit rules, external tools) is introduced.
Debate-style architectures partially mitigate this by averaging
across multiple agents, thereby reducing
$\sigma_{\mathcal{V}}^2 \approx \sigma_{\mathcal{V}}^2 / N$ in an
ideal $N$-agent limit.

\subsection{The Alignment Fiber: Normative Verification}

\textbf{Literature:} \emph{RLHF}, \emph{Constitutional AI}
\cite{bai2022constitutional}.

\begin{example}[RLHF as a Parametric-Verifier GVU]
Standard RLHF fits naturally into the GVU template.
\begin{itemize}
    \item $\mathcal{G}$: the policy $\pi_\theta$ generates traces $y$
    for prompts $x \sim \mu$.
    \item $\mathcal{V}$: a reward model $R_\phi(x,y)$, trained on
    human preference data, produces scores. Combined with a KL penalty
    to a reference policy, this yields an internal potential
    $V(x,y)$.
    \item $\mathcal{U}$: a PPO-style update adjusts $\pi_\theta$
    using $R_\phi$-based advantages, and optionally updates $R_\phi$
    itself.
\end{itemize}
Here the Verifier is parametric and its SNR,
$\mathrm{SNR}(\mathcal{V})$, depends critically on how well $R_\phi$
generalizes beyond the human-labeled dataset. When $R_\phi$ is
undertrained or out-of-distribution (e.g., reward hacking),
$\sigma_{\mathcal{V}}^2$ can become large and the sufficient
condition of Theorem~\ref{thm:variance} may fail, leading to brittle
or unstable improvements (small or even negative effective $\kappa$)
on the Alignment fiber.
\end{example}

\begin{example}[Constitutional AI as an Ensemble Verifier]
In Constitutional AI and related RLAIF methods, human feedback is
replaced or augmented by AI judges guided by a written constitution.
\begin{itemize}
    \item $\mathcal{G}$: a base policy model $\pi_\theta$ produces
    candidate answers $y$ to prompts $x$.
    \item $\mathcal{V}$: one or more judge models evaluate $(x,y)$
    against constitutional principles (e.g., helpfulness,
    harmlessness, honesty). Their scores or pairwise preferences are
    aggregated into a potential $V(x,y)$.
    \item $\mathcal{U}$: the policy is updated by RL on these scores
    or by direct preference optimization on AI-judged pairs, and
    possibly distilled into a student model.
\end{itemize}
When multiple heterogeneous judges are used, $\mathcal{V}$ becomes an
ensemble verifier: averaging or voting across judges tends to reduce
the variance of the verification signal, with
$\sigma_{\mathcal{V}}^2$ decreasing roughly like
$\sigma_{\mathcal{V}}^2 / N$ in an ideal $N$-judge regime. This
improves $\mathrm{SNR}(\mathcal{V})$ relative to a single reward
model and thereby improves the local $\kappa$-slope on the Alignment
fiber. In the diagonal limit where the same model acts as both actor
and judge (up to prompting), additional asymmetries (stricter judge
prompts, temperature settings, or diverse councils) are required to
avoid the Hallucination Barrier.
\end{example}

\subsection{Synthetic Data Bootstrapping: Diagonal GVU}

\textbf{Literature:} \emph{Self-Instruct} \cite{wang2022self}.

\begin{example}[Self-Instruct as Diagonal GVU]
Consider a model $M_\theta$ expanding its own training set.
\begin{itemize}
    \item $\mathcal{G}_\theta$: generates new instruction--response
    pairs (and optionally rationales) from a small seed set.
    \item $\mathcal{V}_\theta$: the same model, with a different
    prompt, filters or ranks examples (e.g., “Is this instruction
    sensible?” “Is this answer correct?”), possibly with additional
    heuristic filters.
    \item $\mathcal{U}_\theta$: fine-tunes $M_\theta$ or a student
    model on the filtered synthetic dataset.
\end{itemize}
This is a diagonal regime with
$\mathcal{G}_\theta \approx \mathcal{V}_\theta \approx \mathcal{U}_\theta$.
Generation and verification noises are tightly coupled, so
$\sigma_{\mathcal{V}} \approx \sigma_{\mathcal{G}}$ and the verifier
inherits the generator's biases. Unless strong external filters or
grounded checks are applied, this setup lies near the Hallucination
Barrier: for realistic step sizes the curvature penalty from noisy
self-judgement can overwhelm the signal, and the effective $\kappa$
may stagnate or drift negative, explaining why purely self-generated
corpora often exhibit semantic drift.
\end{example}

\subsection{Critic-Less Architectures: GRPO}

\textbf{Literature:} \emph{DeepSeek-R1} \cite{deepseek2024}.

\begin{example}[GRPO as a Variance-Reduced GVU Operator]
Group Relative Policy Optimization (GRPO) removes the learned value
function (Critic) used in PPO.
\begin{itemize}
    \item $\mathcal{G}$: for each input $x$, the policy
    $\pi_\theta$ samples a group of $G$ traces
    $\{y_1,\dots,y_G\} \sim \pi_\theta(\cdot \mid x)$.
    \item $\mathcal{V}$: computes ground-truth rewards
    $r_i = R(x,y_i)$ (e.g., correctness) and constructs relative
    advantages via group statistics
    \[
    A_i = \frac{r_i - \mu_{\text{group}}}{\sigma_{\text{group}} + \eps},
    \]
    where $\mu_{\text{group}}$ and $\sigma_{\text{group}}$ are the
    mean and standard deviation across the group. This defines a
    local potential $V(x,y_i) := A_i$.
    \item $\mathcal{U}$: maximizes a PPO-style clipped objective
    using $A_i$ as advantages, without a separate value-network loss.
\end{itemize}
\begin{remark}[Spectral Advantage of GRPO]\label{rem:grpo-spectral}
The empirical success of GRPO can be read through the lens of the
Variance Inequality (Theorem~\ref{thm:variance}) in three ways:
\begin{enumerate}
    \item \textbf{Validation.} In complex reasoning tasks (the Planning
    Fiber), the mapping from initial tokens to final correctness is
    highly chaotic. A parametric Critic network $V_\phi$ attempts to
    predict this value \emph{ex ante}. Often this predictor fails to
    converge or hallucinates, leading to high verification noise
    $\sigma_{\mathcal{V}_\phi}$. By removing the Critic, GRPO ensures
    that $\sigma_{\mathcal{V}}$ is driven only by the ground-truth
    reward variance, which is often lower than the error variance of a
    learned proxy.
    
    \item \textbf{Mechanism.} GRPO substitutes a noisy \emph{prediction}
    of value with the \emph{empirical realization} of value over a
    group. By reinforcing traces relative to the group baseline
    (``winners vs.\ losers'' in the current batch), it extracts a
    positive alignment $\rho > 0$ without requiring a value network
    to generalize across the entire input space $\mathcal{X}$.
    
    \item \textbf{Topological Interpretation.} In our framework, GRPO
    is a GVU operator that computes the update vector $\dot{\theta}$
    using \emph{local batch geometry} (relative differences between
    co-generated traces $\{y_1,\dots,y_G\}$) rather than a global
    value function defined everywhere on $\ThetaMan$. This localization
    reduces $\sigma_{\mathcal{V}}^2$ analytically (roughly by a factor
    $1/G$ under mild assumptions) and thus improves
    $\mathrm{SNR}(\mathcal{V})$ in the sense of
    Theorem~\ref{thm:variance}.
\end{enumerate}
\end{remark}

\end{example}

\section{Operationalization}\label{sec:ops}

To measure $\kappa$ in practice, we propose a differential evaluation protocol.

\begin{enumerate}
\item Fix a base agent $\theta_0$ and a battery $\mathcal{B}$.
\item Allocate a "compute budget" $R$ (e.g., $10^9$ tokens of self-play).
\item Execute the GVU operator $\theta_{i+1} = \mathcal{T}_{\mathrm{GVU}}(\theta_i)$ until cost $R$ is consumed.
\item Measure $\Delta \Phi = \Phi_{\mathcal{B}}(\rho_{\mathcal{B}}(\theta_R))
             - \Phi_{\mathcal{B}}(\rho_{\mathcal{B}}(\theta_0))$.
\item The empirical self-improvement rate is $\hat{\kappa} = \Delta \Phi / R$.
\end{enumerate}

This metric separates "knowledge" (static score) from "autonomy" (derivative). An AGI candidate must demonstrate $\hat{\kappa} > 0$ across the entire moduli space, not just on Math fibers.


\section{Conclusion}\label{sec:concl}

We have provided a unified theory of autonomous self-improvement. By viewing the agent as a flow $\nu_r$ on the moduli space of batteries, we identified the GVU operator as the infinitesimal generator of this flow. 

The Variance Inequality (Theorem \ref{thm:variance}) provides a spectral
constraint on AGI architectures: in our framework, capability cannot be
bootstrapped from noise unless the alignment $\rho$, step size,
curvature, and the joint SNRs of generation and verification fall into a
favorable regime. In practice, many successful self-play architectures
achieve this by arranging for verification to be spectrally easier (that
is, higher effective SNR) than generation.

The literature on Self-Play (SPIN), Reasoning Bootstrapping (STaR), and Reflexion are not distinct algorithms, but topological variations of the same GVU operator acting on different fibers of the moduli space. Achieving true AGI corresponds to closing the loop: constructing a Verifier $\mathcal{V}$ that is universally robust ($\rho \approx 1$) across the entire moduli space, allowing $\kappa > 0$ to be sustained indefinitely.

\bibliographystyle{plain}

\end{document}